\documentclass[journal]{IEEEtran}

\usepackage{graphicx}
\usepackage{amsthm}
\usepackage{amssymb,latexsym,multicol,caption}     
\usepackage{algorithm}  
\usepackage{algorithmic}
\usepackage{epsfig} 
\usepackage{epstopdf}
\usepackage[colorlinks, linkcolor=blue, anchorcolor=red, citecolor=red, urlcolor=black]{hyperref}
\usepackage{color}
\usepackage{multirow}
\usepackage{multicol}
\usepackage{subfigure}
\usepackage{placeins}
\usepackage{bm}
\usepackage{graphics}
\usepackage{footnote}
\usepackage{url}
\usepackage{indentfirst}
\usepackage{longtable}
\usepackage{array}
\usepackage{mdwmath}
\usepackage{mdwtab}
\usepackage{rotating}
\usepackage{color}
\usepackage{morefloats}
\usepackage{slashbox}
\usepackage{mcite}
\usepackage{eqnarray}
\usepackage{cite}
\usepackage{float}
\usepackage{makecell}
\usepackage{amsmath}
\usepackage{threeparttable}
\usepackage{pifont}
\usepackage{cleveref}
\usepackage{balance}
\usepackage{flushend}

\renewcommand{\eqref}[1]{(\ref{#1})}

\newcommand{\algref}[1]{Algorithm \ref{#1}}
\newcommand{\thmref}[1]{Theorem \ref{#1}}

\newcommand{\argmin}{\arg\!\min}

\crefname{equation}{}{}
\crefname{figure}{Fig.}{Figs.}
\crefname{table}{Table}{Tables}
\crefname{section}{Section}{Sections}
\crefname{prop}{Proposition}{Propositions}
\crefname{theorem}{Theorem}{Theorems}
\crefname{lemma}{Lemma}{Lemmas}

\newtheorem{theorem}{Theorem}[section]
\newtheorem{lemma}{Lemma}[section]
\newtheorem{prop}{Proposition}[section]
\newtheorem{definition}{Definition}[section]
\newcommand{\reflemma}[1]{Lemma \ref{#1}}

\newcommand{\vc}[1]{{\textbf{v}}{(#1)}}
\newcommand{\tr}[1]{\textbf{Tr}{(#1)}}

\ifCLASSINFOpdf
\else
\fi

\begin{document}

\title{Two-Dimensional Semi-Nonnegative Matrix Factorization for Clustering}
%


\author{Chong~Peng, Zhilu Zhang, 
        Zhao~Kang, Chenglizhao~Chen, 
        and~Qiang~Cheng,~\IEEEmembership{Senior~Member,~IEEE}
\thanks{C. Peng and C. Chen are with the College of Computer Science and Technology, Qingdao University, Qingdao, Shandong, 266000, China; Z. Kang is with School of Computer Science and Engineering, University of Electronic Science and Technology
of China, Chengdu, Sichuan 611731, China; Q. Cheng is with Institute of Biomedical Informatics \& Department of Computer Science, University of Kentucky, Lexington, KY 40536, USA. E-mail: (pchong1991@163.com, sckangz@gmail.com, qiang.cheng@uky.edu).}
        
}

%
%

\markboth{Journal of \LaTeX\ Class Files,~Vol.~14, No.~8, August~2015}%
{Shell \MakeLowercase{\textit{et al.}}: Bare Demo of IEEEtran.cls for IEEE Journals}
%



\maketitle

\begin{abstract}
In this paper, we propose a new Semi-Nonnegative Matrix Factorization method for 2-dimensional (2D) data, named TS-NMF. It overcomes the drawback of existing methods that seriously damage the spatial information of the data by converting 2D data to vectors in a preprocessing step. In particular, projection matrices are sought under the guidance of building new data representations, such that the spatial information is retained and projections are enhanced by the goal of clustering, which helps construct optimal projection directions. Moreover, to exploit nonlinear structures of the data, manifold is constructed in the projected subspace, which is adaptively updated according to the projections and less afflicted with noise and outliers of the data and thus more representative in the projected space. Hence, seeking projections, building new data representations, and learning manifold are seamlessly integrated in a single model, which mutually enhance other and lead to a powerful data representation. Comprehensive experimental results verify the effectiveness of TS-NMF in comparison with several state-of-the-art algorithms, which suggests high potential of the proposed method for real world applications.
\end{abstract}


%
\IEEEpeerreviewmaketitle

\section{Introduction}
\label{sec_introduction}
Matrix factorization is a powerful way for data representation and has been widely used for many problems in machine learning, data mining, computer vision, and statistical data analysis. Among various factorization algorithms, some have seen widespread successes, such as singular value decomposition (SVD) \cite{duda2012pattern}, and principal component analysis (PCA) \cite{jolliffe2002principal}. 

Recently, a number of relatively new factorization algorithms have been developed to provide improved solutions to some special problems in machine learning \cite{lee1999learning,peng2016fast}. In particular, nonnegative matrix factorization (NMF) \cite{lee1999learning,lee2001algorithms} has drawn considerable attention. NMF represents nonnegative data with nonnegative basis and coefficients, which naturally leads to parts-based representations \cite{lee1999learning}. 
It has been used in many real world applications, such as pattern recognition \cite{li2001learning}, multimedia analysis \cite{cooper2002summarizing}, and text mining \cite{xu2003document}. Recent studies have revealed interesting relationships between NMF and several other methods. For example, spectral clustering (SC) \cite{ng2002spectral} is shown to be equivalent to a weighted version of kernel K-means \cite{dhillon2007weighted} and both of them are particular cases of clustering with NMF under a doubly stochastic constraint \cite{zass2005unifying}; the Kullback-Leibler divergence-based NMF turns out to be equivalent to the probabilistic latent semantic analysis \cite{ding2006nonnegative,hofmann1999probabilistic}, which has been further developed into the fully probabilistic latent Dirichlet allocation model \cite{blei2003latent}.

Semi-NMF extends the repertoire of NMF by removing the non-negativity constraints on the data and basis, which expands the range of applications of NMF. It also strengthens the connections between NMF and K-means \cite{ding2010convex}. It is noted that K-means can be written as a matrix factorization, where the two factor matrices represent the centroids and cluster indicators. Particularly, the centroids can be general whereas the cluster indicators are all nonnegative, which shows the connection between K-means and Semi-NMF. To exploit nonlinear structures of the data, graph-regularized NMF (GNMF) \cite{cai2011graph} and robust manifold NMF (RMNMF) \cite{huang2014robust} incorporate the graph Laplacian to measure nonlinear relationships of the data on manifold. In particular, GNMF including Frobenius-norm and divergence-based formulations, which require the basis and coefficient matrices to be nonnegative; RMNMF removes the constraints on the basis matrix and can be regarded as a variant of Semi-NMF by incorporating a structured sparsity-inducing norm to enhance its robustness.

These methods have been used on 2-dimensional (2D) data such as images, where 2D data are vectorized for further data processing in a preprocessing step. While the vectorization-based Semi-NMF methodology has been growingly useful, it fails to fully exploit the inherent 2D structures and correlations in the 2D data after vectorizing the 2D data. Furthermore, there is empirical evidence that building a model with vectorized high-dimensional features is not effective to filter the noisy or redundant information in the original feature spaces \cite{fu2016tensor}. Besides the way of vectorizing 2D data, tensor based approaches have been proposed. While they may potentially better exploit spatial structures of the 2D data \cite{zhang2015low}, such approaches still have some limitations: They use all features of the data, hence noisy or redundant features may degrade the learning performance. Also, tensor computation and methods usually involve flattening and folding operations, which, more or less, have issues similar to those of vectorization operation and thus might not fully exploit the true structures of the data. Moreover, tensor methods usually suffer from the following major issues: 1) for candecomp/parafac (CP) decomposition based methods, it is generally NP-hard to compute the CP rank \cite{lu2016tensor,kolda2009tensor}; 2) Tucker decomposition is not unique \cite{kolda2009tensor}; 3) the application of a core tensor and a high-order tensor product would incur information loss of spatial details \cite{letexier2008noise}.

To address these limitations, in this paper, we propose a new Semi-NMF-like method for 2D data, where we directly use the original 2D data to help preserve their 2D spatial structures instead of vectorizing them. It is noted that recently there are tensor approaches to retain spatial information for 2D data \cite{cao2013robust,huang2008simultaneous}. However, tensors are usually reduced to matrices for processing. For example, \cite{zhang2015low} organizes different views of the data by a tensor structure, however in each view each sample is still vectorized and the image spatial information is still damaged. In this paper, we directly use 2D inputs whose inherent structure information is emphasized by two projection matrices, which makes our method starkly different from tensor approach. Specifically, we seek optimal projection matrices and building new representations of the data jointly, aiming at enhancing clustering. These projections matrices are optimal in the sense that they project 2D data to the most expressive subspace. Moreover, manifold is taken into consideration to capture nonlinear structures of the data. In our formulation, the manifold is adaptively updated with projection matrices capturing representative information from 2D data, and thus it is less afflicted with noise and outliers. Therefore, this paper seeks optimal projection directions, factors data for new representations, and learns intrinsic manifold structures in a single, seamlessly integrated framework, such that these tasks mutually enhance and lead to improved clustering as well as powerful representations of 2D data. It is noted that, as a special case, our method is applicable to 1-dimensional data. The main contributions of this paper are summarized as follows: 
\begin{itemize}
\item The optimal 2D data projections and an image subspace are sought for learning new representations of the 2D data and clustering 2D matrices.
\item The proposed method is able to retain intrinsic spatial information of 2D data, and alleviate the adverse effect of irrelevant or less important information.
\item Manifold learning is integrated to enhance the capability of exploiting nonlinear structures of the data. The manifold is adaptively updated according to the 2D projections that capture the most expressive information from the data, and the graph is less afflicted with irrelevant or grossly corrupted features.
\item The proposed model enables 2D feature extraction, adaptive manifold learning, and matrix factorization jointly, thus offering a powerful data representation ability.
\item An efficient optimization algorithm is developed with provable mathematical analysis; extensive experimental results verify the effectiveness of the proposed model and algorithm.
\end{itemize}

The rest of this paper is organized as follows. We review related work in \cref{sec_RelatedWork}. Then we present the proposed model in \cref{sec_proposed} and its optimization in \cref{sec_optimization}. We conduct extensive experiments and show the results in \cref{sec_experiments}. Finally, we conclude this paper in \cref{sec_conclusion}.

\section{Related Work}
\label{sec_RelatedWork}

\subsection{Semi-NMF}
\label{sec_related_nmf}
Given data $Y\in\mathcal{R}^{d\times n}$ with $d$ being the dimension of the data and $n$ being the number of samples, the objective of Semi-NMF is
\begin{equation}
\label{eq_obj_seminmf}
\min_{U,V} \|Y-UV^T\|_F^2, \quad s.t. \quad v_{ij}\ge 0,
\end{equation}
where $U = [u_{ij}]= [u_1, u_2, \cdots, u_k]\in\mathcal{R}^{d\times k}$ contains basis in columns and $V = [v_{ij}]=[v_1, v_2, \cdots, v_n] \in\mathcal{R}^{k\times n}$ are the new representations of the data in rows.

\subsection{Graph Laplacian}
\label{sec_related_gnmf}
Graph Laplacian \cite{chung1997spectral} is widely used to incorporate the intrinsic geometrical structure of the data on manifold. In particular, the manifold enforces the smoothness of the data in linear and nonlinear spaces by minimizing
\begin{equation}
\label{eq_manifold}
\begin{aligned}
	&	\frac{1}{2}\sum_{i=1}^{n}\sum_{j=1}^{n} \|v_i-v_j\|_2^2 w_{ij} \\
= 	&	\sum_{j=1}^{n} d_{jj} v_j^T v_j - \sum_{i=1}^{n}\sum_{j=1}^{n} w_{ij} v_i^T v_j,	\\ 
=	& 	\textbf{Tr}(V D V^T) - \textbf{Tr}(V W V^T)	= \textbf{Tr}(V LV^T),
\end{aligned}
\end{equation}
where $\tr{\cdot}$ is the trace operator, $W = [w_{ij}]$ is the weight matrix that measures the pair-wise similarities of original data points, $D = [d_{ij}]$ is a diagonal matrix with $d_{ii} = \sum_{j}w_{ij}$, and $L=D-W$. It is seen that by minimizing \cref{eq_manifold}, we can have a natural effect that if two data points are close in the intrinsic geometry of the data distribution, then their new representations with respect to the new basis, $v_i$ and $v_j$, are also close \cite{cai2011graph}.

\subsection{ 2DPCA }
\label{sec_related_2dpca}
Let $\textbf{X} = \{X_1, X_2, \cdots, X_n\}$ be a collection of images of size $a\times b$, i.e., $X_i\in\mathcal{R}^{a \times b}$, then the 2D covariance matrix of $\bf{X}$ is estimated by $G_t = \frac{1}{n}(\sum_{i=1}^{n}X_i - \sum_{j=1}^{n}X_j)^T (\sum_{i=1}^{n}X_i - \sum_{j=1}^{n}X_j)$. 2DPCA seeks $r$ projection directions by solving the following \cite{yang2004two}:
\begin{equation}
\label{eq_2dpca}
\max_{P^TP=I_r} \tr{P^TG_tP},
\end{equation}
where $P=[p_1,p_2,\cdots,p_r]\in\mathcal{R}^{b\times r}$ contains $r$ orthonormal projection directions and $I_r$ is an identity matrix of size $r\times r$.

\section{Proposed Method}
\label{sec_proposed}
For 2D data \textbf{X}, \cref{eq_obj_seminmf} naturally leads to a formulation as follows:
\begin{equation}
\label{eq_obj_seminmf3}
\min_{\textbf{U},v_{ij}\ge 0} \sum_{i=1}^{n}\|X_i-\sum_{j=1}^{k}U_jv_{ij}\|_F^2,
\end{equation}
where $\textbf{U} = \{U_i\in\mathcal{R}^{a\times b}\}_{i=1}^{k}$ is a set of 2D centroids. It is seen that all elements or features of 2D matrices $X_i$ are used to construct the new representations of the data and the expressiveness of 2D spatial information is not explicitly considered in \cref{eq_obj_seminmf3}. To alleviate this drawback, we propose to better exploit 2D spatial information by building the new representation $V$ with respect to 2D centroids in a projected subspace with the most expressive spatial information:
\begin{equation}
\label{eq_obj_2dseminmf}
\min_{\textbf{U},v_{ij}\ge 0, P^TP = I_r} \sum_{i=1}^{n}\|X_i - \sum_{j=1}^{k}(U_jPP^T) v_{ij}\|_F^2.
\end{equation}
It is noted that in \cref{eq_obj_2dseminmf}, $U_j PP^T$ projects the $j$th centroid to a subspace of rank $r$ with the most expressive information, so that the sum of squared reconstruction errors of 2D matrices $X_i$ from the new basis and new representations can be minimized. As a result, the new representation and the new basis are sought jointly in the projected, most expressive, low-rank subspace to take the advantage of 2D spatial information. Let $C_i = \sum_{j=1}^{k}U_j v_{ij}$, then
\begin{equation}
\begin{aligned}
&	\sum_{i=1}^{n}\left\|X_i - \sum_{j=1}^{k}\left(U_jPP^T\right) v_{ij}\right\|_F^2	\\
=	& \sum_{i=1}^{n}  \Bigg\{ \|X_i\|_F^2 + \|C_iPP^T\|_F^2 - 2\textbf{Tr}(X_i^T C_iPP^T)  \Bigg\} \\
=	& \sum_{i=1}^{n} \Bigg\{ \|X_i\|_F^2 + \|C_iPP^T\|_F^2 - 2\textbf{Tr}(X_i^TC_iPP^T) \\
	& \quad + \|X_iP\|_F^2 + \|X_iP\|_F^2 - 2\textbf{Tr}(P^TX_i^TX_iP)  \Bigg\} 	\\
=	& \sum_{i=1}^{n} \Bigg\{ \|X_iP\|_F^2 + \|C_iP\|_F^2 - 2\textbf{Tr}(P^TX_i^TC_iP) \Bigg\} 	\\
	&+ \sum_{i=1}^{n} \Bigg\{ \|X_i\|_F^2 + \|X_iPP^T\|_F^2 - 2\textbf{Tr}(X_i^TX_iPP^T) \Bigg\} \\
=	& \sum_{i=1}^{n} \Big\|X_iP - \sum_{j=1}^{k}U_j P v_{ij}\Big\|_F^2 + \sum_{i=1}^{n} \Big\|X_i - X_iPP^T\Big\|_F^2
\end{aligned}
\end{equation} 
%
%
where $\textbf{Tr}(\cdot)$ is the trace operator. The second equation is true because $\|X_i P\|_F^2 = \tr{P^T X_i^T X_i P}$. The third equation is true because it can easily verify that $\|X_i P\|_F^2 = \|X_i PP^T\|_F^2$. It is seen that, in the new formulation, the new representation $v_i$ is sought with the projected data $X_jP$'s in the first term, while in the second term the projection ensures that the most expressive information of the data is retained in the subspace given by $PP^T$. With $P^TP = I_r$, it is straightforward that \cref{eq_obj_2dseminmf} can be written as
%
%
\begin{equation}
\label{eq_obj_2dseminmf_written}
\begin{aligned}
&	\min_{ \textbf{U}, V, P} \Bigg\{\sum_{i=1}^{n} \Big\|X_iPP^T  -  \sum_{j=1}^{k}  U_j PP^T  v_{ij} \Big\|_F^2 \\
&	\qquad\qquad\qquad\qquad\qquad + \sum_{i=1}^{n} \|X_i - X_iPP^T\|_F^2 \Bigg\} \\
&	s.t.\quad v_{ij}\ge 0, P^TP = I_r.
\end{aligned}
\end{equation}
It is noted that the first term in \cref{eq_obj_2dseminmf_written} is essentially equivalent to the first term in last equation of \cref{eq_obj_2dseminmf}, but \cref{eq_obj_2dseminmf_written} keeps the physical meanings of $X_iPP^T$. With simple algebra, the second term in \cref{eq_obj_2dseminmf_written} can be written as $\sum_{i=1}^{n} \|X_i - X_iPP^T\|_F^2 = \tr{\sum_{i=1}^{n} X_i^T X_i} - \tr{P^T(\sum_{i=1}^{n} X_i^T X_i)P}$. We omit the constant term $\tr{\sum_{i=1}^{n} X_i^T X_i}$ and introduce a balancing parameter $\lambda_1\ge 0$ to balance the two terms of \cref{eq_obj_2dseminmf_written} to make it more versatile, which gives raise to 
%
%
%
\begin{equation}
\label{eq_obj_2dseminmf_proj}
\begin{aligned}
&\min_{\textbf{U}, V, P} \sum_{i=1}^{n} \Big\|X_iPP^T - \sum_{j=1}^{k}U_j v_{ij} PP^T \Big\|_F^2 - \lambda_1 \textbf{Tr} (P^T G_P P)\\
& s.t.\quad	v_{ij}\ge 0, P^TP = I_r,
\end{aligned}
\end{equation}
where we use the notation of $G_P = \sum_{i=1}^{n}X_i^TX_i$. When $\lambda_1 = 1$, \cref{eq_obj_2dseminmf_proj} falls back to \cref{eq_obj_2dseminmf_written}. It is seen that, by minimizing \cref{eq_obj_2dseminmf_proj}, $P$ is sought so that the data points are projected to the most expressive subspace, aiming at building new, expressive data representations for clustering. Because clustering is performed with projected data, the adverse affects of noise, occlusions or corruptions can be alleviated. Consequently, \cref{eq_obj_2dseminmf_proj} is inherently robust, even though we do not explicitly enforce robustness or use sparsity-inducing norms to measure reconstruction errors. 

%
\cref{eq_obj_2dseminmf_proj} only considers the linear structures of the projected data while overlooking nonlinear ones which usually exist and are important in real world applications. To address this issue, we enforce the smoothness between linear and nonlinear structures on manifold with the following formulation:
\begin{equation}
\label{eq_obj_2dseminmf_proj_graph}
\begin{aligned}
	\min_{ \textbf{U}, v_{ij}\ge 0, P^TP=I_{r} }  & \Big\| \vc{\textbf{X}^{(P)}} - \vc{\textbf{U}^{(P)}}V^T \Big\|_F^2   \\
&  - \lambda_1 \tr{P^T G_P P} + \lambda_2 \tr{V^T L_P V}, 
\end{aligned}
\end{equation}
where $\lambda_2\ge0$ is a balancing parameter. Here, for the ease of notation, we define $\textbf{X}^{(P)} = \{X_iPP^T\}_{i=1}^{n}$, $\textbf{U}^{(P)} = \{U_jPP^T\}_{j=1}^{k}$, and define the operator $\vc{\textbf{M}}=[\vc{M_1}, \cdots, \vc{M_n}]$ to convert a set of 2D inputs, $\textbf{M}$, to a matrix containing each vectorized 2D input $\vc{M_i}$ as a column for ease of notation. Different from \cref{eq_manifold}, we construct the one-to-one similarity matrix $W_P$ using $\vc{\textbf{X}^{(P)}}$ instead of $\vc{\textbf{X}}$, such that the graph Laplacian is adaptively learned with the most expressive features. Correspondingly, $D_P$ and $L_P$ are constructed based on $W_P$ in a way similar to the construction of $D$ and $L$ based on $W$ as in \cref{eq_manifold}. Note that the above defined operators starkly differ from straight vectorization because spatial information has been retained and these operators only provide a simple way for notation without damaging information. It is seen that the tasks of seeking projections, recovering new data representations, and manifold learning mutually enhance each other and lead to a powerful data representation. 

To further enhance the capability of capturing 2D spatial information, we develop the following Two-dimensional Semi-NMF (TS-NMF):
%
%
%
\begin{equation}
\label{eq_obj}
\begin{aligned}
&	\min_{ \textbf{U}, V, P, Q } 
	\Bigg\{\Big\| \vc{\textbf{X}^{(P)}} - \vc{\textbf{U}^{(P)}}V^T  \Big\|_F^2  \\
& \qquad\qquad\qquad + \Big\| \vc{\textbf{X}^{(Q)}} - \vc{\textbf{U}^{(Q)}}V^T \Big\|_F^2 \\
& \qquad\qquad\qquad - \lambda_1  (\tr{P^T G_P P} + \tr{Q^T G_Q Q}) \\
& \qquad\qquad\qquad + \lambda_2 \tr{V^T (L_P + L_Q) V} \Bigg\}	\\
& s.t.\quad v_{ij}\ge 0, P^TP=I_{r}, Q^TQ = I_r,
\end{aligned}
\end{equation}
%
where $Q\in\mathcal{R}^{a\times r}$ contains projection directions to project $\textbf{X}$ on left. Here, we define $\textbf{X}^{(Q)} = \{QQ^TX_i\}_{i=1}^{n}$, $\textbf{U}^{(Q)} = \{QQ^TU_j\}_{j=1}^{k}$, and $G_Q$ is constructed in a similar way to $G_P$ where $X_i^T$'s are used instead of $X_i$'s, and $L_Q=D_Q-W_Q$ is constructed in a similar way to $L_P$ where $(QQ^TX_i)$'s are used instead of $(X_iPP^T)$'s. It is noted that \cref{eq_obj} is not convex. For any solution $\{P,Q,\vc{\textbf{U}},V\}$, $\{P,Q,\vc{\textbf{U}}Z,VZ^{-1}\}$ is also a solution with the same objective value of \cref{eq_obj} with $Z$ being a positive diagonal matrix. Furthermore, the objective value of \cref{eq_obj} can be reduced if $z_{ii}$ increases. To eliminate this uncertainty, in practice one usually requires the Euclidean length of $v_j$ to be 1 \cite{xu2003document,cai2011graph} in a post processing step. In this paper, we also adopt this strategy.

\section{Optimization}
\label{sec_optimization}
In this section, we will develop an efficient optimization algorithm to solve \cref{eq_obj}. In the following, we will present the alternating optimization steps for each variable in detail.

\subsection{Updating $P$}
\label{sec_opt_P}
The subproblem for $P$-minimization is\footnote{\scriptsize Inspired by \cite{wang2014feature}, $\lambda_2 \tr{V^T L_P V}$ is not included in the $P$-minimization problem due the difficulty of writing it as a function of $P$ explicitly. Instead, $L_P$ is fixed when solving $P$ and will be updated accordingly after $P$ is updated. Similar strategy is used for $Q$-minimization.}:
\begin{equation}
\label{eq_sub_P}
\min_{ P^TP = I_r } \left \| \vc{\textbf{X}^{(P)}} - \vc{\textbf{U}^{(P)}}V^T \right \|_F^2 - \lambda_1 \tr{P^T G_P P}.
\end{equation}
With straightforward algebra, \cref{eq_sub_P} can be rewritten as
\begin{equation}
\label{eq_sub_P3}
\begin{aligned}
\textbf{Tr} \left( P^T \left( \sum_{i=1}^{n} \Lambda_i^T \Lambda_i - \lambda_1 G_P \right) P \right),
\end{aligned}
\end{equation}
where $\Lambda_i = X_i - \sum_{j=1}^{k} U_j v_{ij}$. Let $\xi = \tr{G_P}$, it is easy to see that $\sum_{i=1}^{n} \Lambda_i^T \Lambda_i - \lambda_1 G_P + \xi I_b$ is positive definite, hence, according to \cite{yang2004two}, $P$ can be obtained by
\begin{equation}
\label{eq_sol_P_pre}
\begin{aligned}
&	\textbf{eig}_r \left( \sum_{i=1}^{n} \Lambda_i^T \Lambda_i - \lambda_1 G_P + \xi I_b \right) \\
= &	\textbf{eig}_r \left( \sum_{i=1}^{n} \Lambda_i^T \Lambda_i - \lambda_1 G_P \right),
\end{aligned}
\end{equation}
where $\textbf{eig}_r(\cdot)$ returns the eigenvectors of the input matrix corresponding to its smallest $r$ eigenvalues.

\subsection{Updating $Q$}
\label{sec_opt_Q}
The subproblem for $Q$-minimization is:
\begin{equation}
\label{eq_sub_Q}
\min_{ Q^TQ = I_r } \left \| \vc{\textbf{X}^{(Q)}} - \vc{\textbf{U}^{(Q)}}V^T \right \|_F^2 - \lambda_1 \tr{Q^T G_Q Q}.
\end{equation}
Similarly to \cref{eq_sub_P,eq_sub_P3,eq_sol_P_pre}, it is easy to see that $Q$ can be solved by
\begin{equation}
\label{eq_sol_Q_pre}
\textbf{eig}_r \left( \sum_{i=1}^{n} \Lambda_i \Lambda_i^T - \lambda_1 G_Q \right),
\end{equation}
where $\sum_{i=1}^{n} \Lambda_i \Lambda_i^T - \lambda_1 G_Q$ is positive definite.

\subsection{Optimizing $V$}
For convenience of theoretical analysis, we define 
\begin{equation}
\label{eq_def_AB}
\begin{aligned}
A_1 & = \vc{\textbf{U}^{(P)}}^T \vc{\textbf{U}^{(P)}}, A_2 = \vc{\textbf{U}^{(Q)}}^T \vc{\textbf{U}^{(Q)}} \\
B_1 & = \vc{\textbf{X}^{(P)}}^T \vc{\textbf{U}^{(P)}}, B_2 = \vc{\textbf{X}^{(Q)}}^T \vc{\textbf{U}^{(Q)}},
\end{aligned}
\end{equation}
and separate a matrix $M$ into two parts by
\begin{equation}
M_{ij}^+ = (|M_{ij}| + M_{ij})/2,\quad M_{ij}^- = (|M_{ij}| - M_{ij})/2.
\end{equation}
Then the $V$-minimization can be written as
\begin{equation}
\label{eq_sub_v}
\min_{V} F(V) \quad s.t.\quad v_{ij}\ge 0,
\end{equation}
where 
\begin{equation}
\label{eq_dec_proof_obj}
\begin{aligned}
	F(V) 	= 		&	\textbf{Tr}\left( -2V^TB_1^+ + 2V^TB_1^- + VA_1^+V^T 	\right.	\\
		&	- VA_1^-V^T - 2V^TB_2^+ + 2V^TB_2^- + VA_2^+V^T 		\\
		&	- VA_2^-V^T + \lambda_2 V^TD_PV - \lambda_2 V^TW_PV \\
		&	\left.  + \lambda_2 V^TD_QV - \lambda_2 V^TW_QV \right).
\end{aligned}
\end{equation}
Then, $V$ is updated by:
\begin{equation}
\label{eq_update_v}
v_{ij} \leftarrow v_{ij}
\sqrt{
\frac{
(B_1^{+}+ B_2^{+} + V(A_1^{-} + A_2^{-}) + \lambda_2 (W_P + W_Q) V)_{ij}}
{(B_1^{-} + B_2^{-} + V(A_1^{+} + A_2^{+}) + \lambda_2 (D_P + D_Q) V)_{ij}}
}.
\end{equation}
Regarding \cref{eq_sub_v,eq_dec_proof_obj,eq_update_v}, similar to the conclusion in \cite{ding2010convex}, we have the following theorem:
\begin{theorem}
\label{thm_decrease_V}
Fixing all other variables, the value of F(V) in \cref{eq_dec_proof_obj} is monotonically non-increasing under the updating \cref{eq_update_v}. Furthermore, the limiting solution of \cref{eq_update_v} satisfies KKT condition.
\end{theorem}
The proof of \cref{thm_decrease_V} is provided in the Appendix. It is noted that \cref{eq_update_v} provides an iterative way to solve \cref{eq_sub_v}, which requires an inner loop for optimization. However, in a way similar to NMF \cite{lee1999learning}, GNMF \cite{cai2011graph}, and Semi-NMF \cite{ding2010convex}, we do not require an exact solution to the subproblem \cref{eq_sub_v}. Instead, \cref{eq_update_v} is performed once to solve \cref{eq_sub_v}. Similar idea is also found in \cite{lin2010augmented}, where exact solutions are not required for intermediate updating.

\subsection{Optimizing $\textbf{U}$}
The subproblem associated with $\textbf{U}$-minimization is
\begin{equation}
\label{eq_sub_U}
\begin{aligned}
\min_{ \textbf{U} }	& \left \| \vc{\textbf{X}^{(P)}} - \vc{\textbf{U}^{(P)}}V^T \right \|_F^2 + \left \| \vc{\textbf{X}^{(Q)}} - \vc{\textbf{U}^{(Q)}}V^T \right \|_F^2.
\end{aligned}
\end{equation}
We investigate the two terms separately. The first term is minimized when it satisfies
\begin{equation}
\vc{\textbf{U}^{(P)}} = \vc{\textbf{X}^{(P)}} V(V^TV)^{-1},
\end{equation}
which is equivalent to the following condition
\begin{equation}
\label{eq_u}
U_i PP^T = \left( \sum_{j=1}^{n} X_j (V(V^TV)^{-1})_{ji} \right) PP^T.
\end{equation}
It is seen that there are infinitely many choices for $U_i$ to meet the above condition, e.g., any $U_i$ such that $U_i - \sum_{j=1}^{n} X_j(V(V^TV)^{-1})_{ji}$ is in the null space of $PP^T$. Here, we use the simplest way to meet this requirement by requiring
\begin{equation}
\label{eq_update_u}
\textbf{U} = \Big\{U_i = \sum_{j=1}^{n} X_j (V(V^TV)^{-1})_{ji}\Big\}_{i=1}^{k}.
\end{equation}
Similarly, we see that the second term in \cref{eq_sub_U} can be simultaneously minimized by \cref{eq_update_u}. Therefore, we adopt \cref{eq_update_u} to update \textbf{U}. Here, it is noted that $V^TV$ is usually invertible and computationally tractable due to its small size. Otherwise, pseudo-inverse is used as in \cite{ding2010convex}.

Finally, we adjust $\textbf{U}$ and $V$ as follows, such that $\vc{\textbf{U}}V^T$ does not change:
\begin{equation}
\label{eq_adjust}
v_{jk}\leftarrow v_{jk}/\|v_k\|^2, \quad \vc{\textbf{U}}_{ik}\leftarrow \vc{\textbf{U}}_{ik}\|v_k\|^2.
\end{equation}
Then standard K-means is applied to V to obtain cluster indicators. We summarize the overall procedure in \algref{alg_tsnmf}.

\begin{algorithm}[h]
\caption{ TS-NMF for Clustering } 
\vspace{1mm}
\begin{algorithmic}[1] 
\STATE \textbf{Input}: $\textbf{X}$, $\lambda_1$, $\lambda_2$, $r$, $t_{max}$
\STATE \textbf{Initialize:} $\textbf{U}^0$, $V^0$, $P^0$, $Q^0$, $t=0$. 
\REPEAT
\STATE Update $P$ and $Q$ by \cref{eq_sol_P_pre,eq_sol_Q_pre}, respectively;
\STATE Update $L_P$ and $L_Q$ by \cref{eq_manifold} using $\vc{\textbf{X}^{(P)}}$ and $\vc{\textbf{X}^{(Q)}}$;
\STATE Update $V$ and \textbf{U} by \cref{eq_sub_v,eq_update_u}, respectively;
\STATE $t=t+1$.
\UNTIL $t\geq t_{max}$ or convergence
\STATE Adjust $\textbf{U}$ and $V$ according to \cref{eq_adjust}, and apply standard K-means to $V$
\STATE \textbf{Output}: Predicted class indicators 
\vspace{1mm} 
\end{algorithmic} 
\label{alg_tsnmf}
\end{algorithm}

\subsection{Complexity Analysis}
\label{sec_complexity}
Because multiplications dominate the complexity, we only count multiplications. Given that $d = ab$, $a\approx b$, $a,b \ge r$, $n\gg k$, let $T$ be the total number of iterations for \algref{alg_tsnmf}, then the total complexity of \algref{alg_tsnmf} is $O(Tn^2d + Tnd^{\frac{3}{2}})$. It is similar to GNMF and RMNMF. The complexity mainly comes from the updating of graph Laplacian matrices with complexity $O(n^2d)$ per iteration. Fortunately, it can be easily parallel for this step per iteration, and thus it is not a bottleneck for real world applications.

\section{Experiments}
\label{sec_experiments}
To demonstrate the effectiveness of TS-NMF, in this section, we present the comprehensive experimental results in comparison with several state-of-the-art algorithms. The performances are measured based on three evaluation metrics including clustering accuracy (ACC), normalized mutual information (NMI), and purity, whose details can be found in \cite{huang2014robust,peng2017nonnegative}. In the following, we will briefly introduce the benchmark data sets, the baseline methods in comparison, and present the experimental results in detail. For purpose of reproducibility, we provide the data and codes at xxxx.

\subsection{Benchmark Data Sets}
\label{sec_exp_data}
We use seven data sets in the experiment, which are briefly described as follows: 1) \textbf{Yale} \cite{belhumeur1997eigenfaces}. It contains 165 gray scale images of 15 persons with 11 images of size 32$\times$32 per person. 2) \textbf{Extended Yale B} (EYaleB) \cite{georghiades2001few}. This data set has 38 persons and around 64 face images under different illuminations per each person. The images were cropped to 192$\times$168 and were resized to 32$\times$32 in our experiments. 3) \textbf{ORL} \cite{samaria1994parameterisation}. This data has 40 individuals and 10 images were taken at different times, with varying facial expressions, facial details, and lighting conditions per each individual. Each image has 32$\times$32 pixels. 4) \textbf{JAFFE} \cite{lyons1998japanese}. 10 Japanese female models posed 7 facial expressions and 213 images were collected. Each image has been rated on 6 motion adjectives by 60 Japanese subjects. 5) \textbf{PIX} \cite{hond1997distinctive}. 100 gray scale images of $100\times100$ pixels from 10 objects were collected. 6) \textbf{Semeion}. 1,593 handwritten digits written by around 80 persons were collected. These images were scanned, stretched into size 16$\times$16. 



\begin{table*}[!tb] 
	\Huge
	\centering
	\caption{Clustering Performance on EYaleB}
	\resizebox{0.8\textwidth}{!}{
		\begin{tabular}{ |c| |c |c |c |c |c |c | c| c| c| c| }
			\hline
			\multirow{2}{1cm}{\centering N}& \multicolumn{10}{c|}{Accuracy (\%)}\\	
			\cline{2-11}
			\multirow{2}{1cm}{}
			& K-Means  			& PCA				& RPCA		 		& 2DPCA				& NMF 				& SC				& GNMF 				& RMNMF				&	Semi-NMF		& TS-NMF		\\\hline
			5	& 23.01$\pm$00.80	& 23.38$\pm$00.68	& 23.01$\pm$00.99	& 23.32$\pm$01.06	& 23.79$\pm$01.61	& ---------------	& 41.97$\pm$09.86	& 28.82$\pm$03.98	& 24.20$\pm$02.27	& \bf{77.27$\pm$11.23}	\\
			10	& 13.84$\pm$01.19	& 13.46$\pm$00.67	& 13.93$\pm$01.03	& 13.89$\pm$00.81	& 14.66$\pm$00.39	& ---------------	& 30.09$\pm$05.00	& 22.67$\pm$02.02	& 16.88$\pm$02.66	& \bf{66.02$\pm$07.15}	\\
			15	& 11.46$\pm$01.01	& 10.60$\pm$00.47	& 10.97$\pm$00.77	& 11.13$\pm$00.94	& 11.38$\pm$00.40	& ---------------	& 23.80$\pm$06.49	& 20.98$\pm$01.98	& 13.66$\pm$01.70	& \bf{61.41$\pm$05.86}	\\
			20	& 10.69$\pm$01.08	& 09.65$\pm$00.51	& 09.85$\pm$00.80	& 10.48$\pm$00.79	& 09.77$\pm$00.43	& ---------------	& 21.17$\pm$01.96	& 19.64$\pm$01.19	& 13.27$\pm$01.03	& \bf{58.92$\pm$05.04}	\\
			25	& 09.35$\pm$01.05	& 07.70$\pm$00.47	& 08.37$\pm$00.38	& 08.78$\pm$00.72	& 08.54$\pm$00.22	& ---------------	& 15.96$\pm$01.99	& 17.75$\pm$01.30	& 10.07$\pm$00.61	& \bf{55.53$\pm$04.68}	\\
			30	& 08.48$\pm$00.71	& 07.33$\pm$00.24	& 08.14$\pm$00.66	& 08.61$\pm$01.02	& 07.88$\pm$00.17	& ---------------	& 16.47$\pm$00.98	& 17.21$\pm$01.10	& 10.26$\pm$00.77	& \bf{55.29$\pm$03.18}	\\
			35	& 08.85$\pm$00.75	& 06.63$\pm$00.28	& 08.67$\pm$00.80	& 08.31$\pm$00.48	& 07.12$\pm$00.16	& ---------------	& 14.64$\pm$00.76	& 16.48$\pm$01.23	& 09.85$\pm$00.76	& \bf{53.73$\pm$02.64}	\\
			38	& 08.53				& 06.59				& 08.99				& 08.33				& 07.08				& ---------------	& 16.16				& 16.86				& 08.70				& \bf{56.84}			\\	\hline
			Average	
			& 11.78				& 10.67				& 11.49				& 11.61				& 11.28				& ---------------	& 22.53				& 20.05				& 13.36				& \bf{60.63}			\\ \hline
			
			\hline
			\multirow{2}{1cm}{\centering N}& \multicolumn{10}{c|}{Normalized Mutual Information (\%)}\\	
			\cline{2-11}
			\multirow{2}{1cm}{}
			& K-Means  			& PCA				& RPCA		 		& 2DPCA				& NMF 				& SC				& GNMF 				& RMNMF				&	Semi-NMF		& TS-NMF		\\\hline
			5	& 00.94$\pm$00.73	& 00.78$\pm$00.44	& 00.74$\pm$00.35	& 00.95$\pm$00.87	& 01.25$\pm$01.61	& ---------------	& 30.22$\pm$15.58	& 06.49$\pm$04.14	& 02.73$\pm$03.44	& \bf{71.33$\pm$08.35}	\\
			10	& 02.39$\pm$02.21	& 01.56$\pm$00.44	& 02.17$\pm$01.15	& 02.24$\pm$00.96	& 03.07$\pm$00.61	& ---------------	& 29.67$\pm$06.65	& 15.34$\pm$01.93	& 06.94$\pm$04.88	& \bf{64.72$\pm$05.26}	\\
			15	& 03.96$\pm$01.48	& 02.77$\pm$00.68	& 03.31$\pm$01.35	& 03.70$\pm$01.63	& 04.05$\pm$00.63	& ---------------	& 28.01$\pm$11.77	& 19.89$\pm$03.74	& 08.49$\pm$03.26	& \bf{64.98$\pm$05.78}	\\
			20	& 06.91$\pm$01.59	& 05.02$\pm$01.17	& 05.50$\pm$01.99	& 06.61$\pm$01.49	& 05.54$\pm$00.56	& ---------------	& 27.62$\pm$03.08	& 23.11$\pm$01.25	& 11.99$\pm$01.67	& \bf{63.68$\pm$02.96}	\\
			25	& 07.00$\pm$01.42	& 04.18$\pm$01.05	& 05.41$\pm$00.88	& 05.93$\pm$01.49	& 06.16$\pm$00.39	& ---------------	& 21.60$\pm$03.34	& 24.04$\pm$01.15	& 10.23$\pm$01.69	& \bf{62.48$\pm$03.21}	\\
			30	& 07.79$\pm$01.19	& 05.51$\pm$00.41	& 07.49$\pm$01.02	& 07.71$\pm$01.01	& 07.34$\pm$00.28	& ---------------	& 23.79$\pm$00.89	& 26.28$\pm$01.09	& 12.74$\pm$01.58	& \bf{62.26$\pm$02.40}	\\
			35	& 10.04$\pm$01.44	& 05.81$\pm$00.55	& 09.41$\pm$01.48	& 09.56$\pm$00.74	& 08.08$\pm$00.24	& ---------------	& 25.36$\pm$01.27	& 27.67$\pm$01.58	& 14.08$\pm$01.14	& \bf{61.12$\pm$02.31}	\\
			38	& 10.51				& 06.13				& 10.42				& 10.26				& 08.79				& ---------------	& 25.86				& 28.46				& 13.13				& \bf{63.63}			\\	\hline
			Average	
			& 06.19				& 03.97				& 05.56				& 05.87				& 05.54				& ---------------	& 26.52				& 21.41				& 10.04				& \bf{64.28}			\\	\hline

		\end{tabular}
	}
	\label{tab_per_eyaleb}
\end{table*}

\begin{table*}[!tb] 
	\Huge
	\centering
	\caption{Clustering Performance on ORL }
	\resizebox{0.8\textwidth}{!}{
		\begin{tabular}{ |c| |c |c |c |c |c |c | c| c| c| c| }
			\hline
			\multirow{2}{1cm}{\centering N}& \multicolumn{10}{c|}{Accuracy (\%)}\\	
			\cline{2-11}
			\multirow{2}{1cm}{}
			& K-Means  			& PCA				& RPCA		 		& 2DPCA				& NMF 				& SC				& GNMF 				& RMNMF				& Semi-NMF			& TS-NMF	\\	\hline
			5	& 79.00$\pm$13.54	& 81.00$\pm$15.36	& 81.00$\pm$14.06	& 80.20$\pm$13.35	& 40.40$\pm$04.97	& 59.60$\pm$08.37	& 81.00$\pm$14.43	& 76.80$\pm$14.79	& 74.80$\pm$11.78	& \bf{81.60$\pm$08.42}	\\
			10	& 62.10$\pm$08.67	& 64.00$\pm$07.16	& 68.20$\pm$09.72	& 65.80$\pm$09.70	& 11.00$\pm$00.00	& 37.00$\pm$08.75	& 71.20$\pm$08.18	& 66.70$\pm$05.81	& 65.80$\pm$06.32	& \bf{76.30$\pm$06.83}	\\
			15	& 62.53$\pm$05.99	& 61.13$\pm$04.49	& 63.53$\pm$06.80	& 64.53$\pm$05.78	& 29.20$\pm$01.80	& 28.67$\pm$03.50	& 67.80$\pm$05.99	& 66.60$\pm$03.01	& 68.27$\pm$06.90	& \bf{73.27$\pm$06.27}	\\
			20	& 57.80$\pm$06.21	& 58.15$\pm$05.68	& 60.90$\pm$05.78	& 61.80$\pm$06.54	& 26.55$\pm$02.77	& 25.80$\pm$01.70	& 65.50$\pm$07.86	& 61.10$\pm$03.41	& 63.05$\pm$03.86	& \bf{70.85$\pm$06.75}	\\
			25	& 57.24$\pm$03.21	& 57.04$\pm$02.24	& 59.16$\pm$03.08	& 60.12$\pm$03.61	& 24.08$\pm$01.16	& 23.60$\pm$01.74	& 63.96$\pm$05.28	& 62.72$\pm$03.66	& 61.16$\pm$04.06	& \bf{68.08$\pm$03.11}	\\
			30	& 55.63$\pm$03.17	& 52.53$\pm$02.59	& 57.97$\pm$02.71	& 57.40$\pm$03.25	& 22.50$\pm$01.00	& 22.17$\pm$01.22	& 62.23$\pm$03.03	& 58.57$\pm$03.81	& 60.10$\pm$04.66	& \bf{67.40$\pm$03.64}	\\
			35	& 52.83$\pm$02.54	& 50.49$\pm$02.74	& 55.11$\pm$03.29	& 56.31$\pm$04.72	& 21.31$\pm$00.65	& 20.71$\pm$00.79	& 59.60$\pm$04.22	& 56.60$\pm$02.79	& 57.91$\pm$03.26	& \bf{64.20$\pm$02.63}	\\
			40	& 53.50				& 44.00				& 63.00				& 55.00				& 20.25				& 20.25				& 55.75				& 56.25				& 57.75				& \bf{68.00}			\\	\hline
			Average		
			& 60.09				& 58.54				& 63.61				& 62.56				& 24.41				& 29.73				& 65.88				& 63.17				& 63.61				& \bf{71.21}			\\	\hline
			\hline
			\multirow{2}{1cm}{\centering N}& \multicolumn{10}{c|}{Normalized Mutual Information (\%)}\\	
			\cline{2-11}
			\multirow{2}{1cm}{}
			& K-Means  			& PCA				& RPCA		 		& 2DPCA				& NMF 				& SC				& GNMF 				& RMNMF				& Semi-NMF			& TS-NMF	\\	\hline
			5	& 73.72$\pm$15.91	& 77.97$\pm$15.06	& 77.51$\pm$13.57	& 74.43$\pm$16.23	& 23.68$\pm$06.98	& 53.17$\pm$09.32	& 77.63$\pm$14.69	& 71.64$\pm$15.32	& 72.33$\pm$10.60	& \bf{77.48$\pm$06.51}	\\
			10	& 67.81$\pm$08.47	& 71.63$\pm$05.75	& 74.16$\pm$07.61	& 72.39$\pm$07.42	& 11.73$\pm$00.00	& 38.64$\pm$09.37	& 77.19$\pm$06.29	& 70.56$\pm$03.50	& 73.46$\pm$05.63	& \bf{80.51$\pm$05.15}	\\
			15	& 72.91$\pm$05.39	& 72.27$\pm$03.98	& 73.69$\pm$04.91	& 74.73$\pm$04.78	& 37.84$\pm$01.32	& 35.32$\pm$04.15	& 77.28$\pm$04.46	& 73.67$\pm$02.22	& 77.50$\pm$05.86	& \bf{80.95$\pm$04.26}	\\
			20	& 71.60$\pm$04.32	& 71.91$\pm$04.57	& 73.40$\pm$03.73	& 74.23$\pm$04.82	& 40.17$\pm$02.41	& 38.02$\pm$01.01	& 76.93$\pm$05.26	& 72.62$\pm$02.62	& 74.95$\pm$02.39	& \bf{80.42$\pm$04.02}	\\
			25	& 72.45$\pm$01.91	& 71.63$\pm$01.49	& 73.38$\pm$01.64	& 73.89$\pm$02.32	& 41.04$\pm$01.31	& 39.22$\pm$01.46	& 77.73$\pm$02.67	& 75.15$\pm$02.24	& 74.99$\pm$03.18	& \bf{79.07$\pm$02.08}	\\
			30	& 72.53$\pm$01.90	& 71.20$\pm$01.99	& 73.90$\pm$02.14	& 73.79$\pm$01.89	& 42.42$\pm$01.14	& 40.53$\pm$00.89	& 77.02$\pm$02.35	& 73.34$\pm$02.33	& 75.48$\pm$03.30	& \bf{80.34$\pm$02.03}	\\
			35	& 70.71$\pm$01.40	& 70.60$\pm$01.50	& 72.12$\pm$01.75	& 73.32$\pm$02.95	& 43.01$\pm$00.85	& 41.48$\pm$00.66	& 75.51$\pm$02.27	& 72.43$\pm$01.70	& 74.54$\pm$02.42	& \bf{79.01$\pm$01.61}	\\
			40	& 71.82				& 69.07				& 72.35				& 74.07				& 43.01				& 42.64				& 74.72				& 73.03				& 75.32				& \bf{81.27}			\\	\hline
			Average		
			& 71.69				& 72.03				& 73.81				& 73.86				& 35.36				& 41.13				& 76.75				& 72.81				& 74.82				& \bf{79.88}			\\	\hline
		\end{tabular}
	}
	\label{tab_per_orl}
\end{table*}

\begin{table*}[!tb] 
	\Huge
	\centering
	\caption{Clustering Performance on Semeion }
	\resizebox{0.8\textwidth}{!}{
		\begin{tabular}{ |c| |c |c |c |c |c |c | c| c| c| c|  }
			\hline
			\multirow{2}{1cm}{\centering N}& \multicolumn{10}{c|}{Accuracy (\%)}\\	
			\cline{2-11}
			\multirow{2}{1cm}{}
			& K-Means  			& PCA				& RPCA		 		& 2DPCA				& NMF 				& SC				& GNMF 					& RMNMF				& Semi-NMF			& TS-NMF	\\ \hline
			2	& 89.05$\pm$10.42	& 90.42$\pm$08.58	&  92.15$\pm$06.73	& 90.66$\pm$08.57	& 69.39$\pm$09.88	& ---------------	& \bf{95.46$\pm$05.63}	& 87.58$\pm$10.64	& 87.41$\pm$10.76	& 95.18$\pm$05.84		\\
			3	& 82.97$\pm$08.58	& 83.13$\pm$08.31	&  83.57$\pm$08.32	& 83.22$\pm$08.46	& 50.15$\pm$09.24	& ---------------	& 85.41$\pm$17.18		& 78.23$\pm$09.17	& 79.83$\pm$09.72	& \bf{86.36$\pm$08.10}	\\
			4	& 75.41$\pm$11.13	& 77.55$\pm$07.30	&  75.13$\pm$11.56	& 75.79$\pm$09.24	& 43.43$\pm$06.98	& ---------------	& 77.90$\pm$13.92		& 65.22$\pm$07.80	& 71.45$\pm$09.13	& \bf{83.23$\pm$11.88}	\\
			5	& 75.16$\pm$07.44	& 77.55$\pm$06.03	&  74.23$\pm$07.17	& 75.49$\pm$09.24	& 39.09$\pm$04.57	& ---------------	& 82.76$\pm$08.45		& 62.33$\pm$07.31	& 71.69$\pm$08.37	& \bf{84.73$\pm$09.07}	\\
			6	& 63.45$\pm$10.28	& 65.81$\pm$10.08	&  65.28$\pm$08.80	& 64.84$\pm$09.07	& 33.62$\pm$03.05	& ---------------	& 71.47$\pm$11.65		& 54.67$\pm$06.88	& 68.07$\pm$05.96	& \bf{73.59$\pm$09.66}	\\
			7	& 63.16$\pm$06.17	& 69.06$\pm$05.73	&  63.83$\pm$05.63	& 63.52$\pm$07.17	& 27.64$\pm$02.11	& ---------------	& 63.88$\pm$05.62		& 52.94$\pm$06.03	& 64.98$\pm$05.99	& \bf{74.83$\pm$05.79}	\\
			8	& 67.90$\pm$07.40	& 69.11$\pm$05.11	&  67.09$\pm$06.39	& 64.10$\pm$05.46	& 26.45$\pm$00.86	& ---------------	& 69.37$\pm$07.02		& 48.23$\pm$04.31	& 64.13$\pm$04.94	& \bf{75.62$\pm$07.84}	\\
			9	& 61.38$\pm$05.31	& 61.36$\pm$05.41	&  59.94$\pm$05.76	& 62.15$\pm$03.31	& 24.69$\pm$00.99	& ---------------	& 61.34$\pm$02.60		& 44.90$\pm$02.77	& 57.98$\pm$02.51	& \bf{73.62$\pm$08.18}	\\
			10	& 54.55				& 64.28				&  54.36			& 60.33				& 22.91				& ---------------	& 63.03					& 43.57				& 60.14				& \bf{71.00}			\\	\hline
			Average	
			& 70.34				& 73.14				&  70.62			& 71.12				& 37.49				& ---------------	& 74.51					& 59.74				& 69.52				& \bf{79.79}			\\ \hline
			\hline
			\multirow{2}{1cm}{\centering N}& \multicolumn{10}{c|}{Normalized Mutual Information (\%)}\\	
			\cline{2-11}
			\multirow{2}{1cm}{}
			& K-Means  			& PCA				& RPCA		 		& 2DPCA				& NMF 				& SC				& GNMF 					& RMNMF				& Semi-NMF			& TS-NMF	\\ \hline
			2	& 61.08$\pm$29.93	& 63.41$\pm$26.07	&  67.77$\pm$21.94	& 64.30$\pm$25.81	& 14.15$\pm$09.50	& ---------------	& 77.95$\pm$19.53		& 55.48$\pm$28.88	& 56.09$\pm$28.43	& \bf{78.37$\pm$17.47}	\\
			3	& 58.78$\pm$11.94	& 58.78$\pm$11.50	&  60.42$\pm$10.95	& 59.06$\pm$12.09	& 13.60$\pm$11.14	& ---------------	& \bf{70.30$\pm$15.81}	& 50.39$\pm$12.33	& 53.55$\pm$12.81	& 68.73$\pm$14.77		\\
			4	& 58.02$\pm$09.08	& 58.83$\pm$07.07	&  58.92$\pm$09.32	& 56.13$\pm$11.11	& 17.44$\pm$09.41	& ---------------	& 66.30$\pm$12.69		& 44.83$\pm$06.88	& 51.55$\pm$05.43	& \bf{72.96$\pm$12.13}	\\
			5	& 61.16$\pm$06.75	& 61.29$\pm$07.05	&  60.33$\pm$06.31	& 61.41$\pm$07.31	& 21.79$\pm$05.23	& ---------------	& 73.16$\pm$06.35		& 43.45$\pm$07.15	& 53.71$\pm$08.96	& \bf{74.27$\pm$10.81}	\\
			6	& 54.71$\pm$08.23	& 54.04$\pm$07.73	&  55.06$\pm$07.93	& 55.34$\pm$07.59	& 18.42$\pm$03.73	& ---------------	& 62.98$\pm$11.05		& 39.81$\pm$06.31	& 51.44$\pm$05.82	& \bf{64.29$\pm$10.86}	\\
			7	& 54.38$\pm$04.31	& 55.00$\pm$05.03	&  55.10$\pm$04.31	& 54.71$\pm$05.78	& 16.17$\pm$04.21	& ---------------	& 58.31$\pm$03.67		& 41.71$\pm$04.53	& 52.64$\pm$04.71	& \bf{65.87$\pm$05.78}	\\
			8	& 58.94$\pm$04.37	& 56.68$\pm$03.34	&  58.43$\pm$03.70	& 56.69$\pm$03.66	& 18.29$\pm$01.67	& ---------------	& 64.05$\pm$05.99		& 39.51$\pm$03.19	& 53.46$\pm$04.09	& \bf{68.93$\pm$06.20}	\\
			9	& 55.05$\pm$03.43	& 53.38$\pm$03.08	&  54.40$\pm$03.55	& 55.44$\pm$03.30	& 17.40$\pm$01.01	& ---------------	& 59.79$\pm$02.67		& 36.52$\pm$02.66	& 49.92$\pm$02.11	& \bf{69.63$\pm$05.60}	\\
			10	& 51.67				& 53.19				&  51.18			& 55.27				& 16.88				& ---------------	& \bf{58.88}			& 35.44				& 52.34				& \bf{63.53}					\\	\hline
			Average	
			& 57.09				& 57.18				&  57.96			& 57.60				& 17.13				& ---------------	& 65.75					& 43.02				& 52.74			 	& \bf{69.62}			\\	\hline
		\end{tabular}
	}
	\label{tab_per_semeion}
\end{table*}

\begin{table*}[!tb] 
	\Huge
	\centering
	\caption{Clustering Performance on JAFFE }
	\resizebox{0.8\textwidth}{!}{
		\begin{tabular}{ |c| |c |c |c |c |c |c | c| c| c| c| }
			\hline
			\multirow{2}{1cm}{\centering N}& \multicolumn{10}{c|}{Accuracy (\%)}\\	
			\cline{2-11}
			\multirow{2}{1cm}{}
			& K-Means  				& PCA					& RPCA		 			& 2DPCA					& NMF 				& SC					& GNMF 					& RMNMF					& Semi-NMF			& TS-NMF		\\ \hline
			2	& \bf{100.0$\pm$00.00}	& \bf{100.0$\pm$00.00}	& \bf{100.0$\pm$00.00}	& \bf{100.0$\pm$00.00}	& 64.95$\pm$09.34	& \bf{100.0$\pm$00.00}	& \bf{100.0$\pm$00.00}	& \bf{100.0$\pm$00.00}	& 99.75$\pm$00.79	& \bf{100.0$\pm$00.00}	\\
			3	& 98.40$\pm$01.86		& 98.55$\pm$01.93		& \bf{100.0$\pm$00.00}	& 98.40$\pm$01.86		& 53.09$\pm$07.96	& 84.97$\pm$19.34		& 99.84$\pm$00.51		& 97.62$\pm$01.86		& 95.44$\pm$06.40	& 99.84$\pm$00.51		\\
			4	& 99.30$\pm$01.83		& 99.30$\pm$01.83		& 99.19$\pm$02.57		& 97.79$\pm$05.27		& 51.02$\pm$05.49	& 72.55$\pm$11.74		& 99.42$\pm$01.26		& 98.83$\pm$01.73		& 95.45$\pm$05.99	& \bf{99.53$\pm$01.12}	\\
			5	& 98.68$\pm$02.00		& 98.67$\pm$02.00		& 98.87$\pm$02.42		& 98.58$\pm$01.95		& 45.09$\pm$04.83	& 74.08$\pm$10.70		& 99.15$\pm$01.37		& 97.46$\pm$03.09		& 95.34$\pm$05.04	& \bf{99.62$\pm$00.67}	\\
			6	& 95.97$\pm$04.13		& 97.10$\pm$02.03		& \bf{99.38$\pm$01.95}	& 93.31$\pm$04.63		& 40.64$\pm$04.38	& 63.27$\pm$10.08		& 97.25$\pm$06.55		& 95.14$\pm$04.07		& 90.40$\pm$06.39	& \bf{99.38$\pm$01.14}	\\
			7	& 95.65$\pm$06.03		& 96.79$\pm$02.24		& 97.33$\pm$02.63		& 93.51$\pm$05.66		& 38.53$\pm$05.85	& 59.00$\pm$09.51		& 96.48$\pm$06.69		& 90.24$\pm$06.90		& 92.66$\pm$05.58	& \bf{99.13$\pm$01.05}	\\
			8	& 91.97$\pm$06.28		& 95.94$\pm$01.31		& 97.05$\pm$02.19		& 91.37$\pm$04.63		& 36.14$\pm$03.36	& 61.64$\pm$05.38		& 93.68$\pm$08.43		& 91.63$\pm$05.58		& 94.27$\pm$05.43	& \bf{99.12$\pm$00.89}	\\
			9	& 91.87$\pm$04.43		& 94.16$\pm$01.61		& 94.63$\pm$01.25		& 89.94$\pm$04.77		& 35.96$\pm$03.10	& 61.37$\pm$11.03		& 95.30$\pm$07.35		& 90.73$\pm$07.06		& 88.83$\pm$09.06	& \bf{99.01$\pm$00.76}	\\
			10	& 84.04					& 86.85					& 95.77					& 92.02					& 33.80				& 57.75					& 97.65					& 95.77					& 95.77				& \textbf{100.0}		\\	\hline
			Average	
			& 95.10					& 96.37					& 98.02					& 94.99					& 44.36				& 70.51					& 97.64					& 95.27					& 94.21				& \bf{99.51}			\\	\hline
			\hline
			\multirow{2}{1cm}{\centering N}& \multicolumn{10}{c|}{Normalized Mutual Information (\%)}\\	
			\cline{2-11}
			\multirow{2}{1cm}{}
			& K-Means  				& PCA					& RPCA		 			& 2DPCA					& NMF 				& SC					& GNMF 					& RMNMF					& Semi-NMF			& TS-NMF	\\ \hline
			2	& \bf{100.0$\pm$00.00}	& \bf{100.0$\pm$00.00}	& \bf{100.0$\pm$00.00}	& \bf{100.0$\pm$00.00}	& 13.90$\pm$14.09	& \bf{100.00$\pm$00.00}	& \bf{100.0$\pm$00.00}	& \bf{100.0$\pm$00.00}	& 98.55$\pm$04.59	& \bf{100.0$\pm$00.00}	\\
			3	& 94.88$\pm$05.88		& 95.46$\pm$06.09		& \bf{100.0$\pm$00.00}	& 94.88$\pm$05.88		& 20.02$\pm$11.58	& 70.23$\pm$19.02		& 99.41$\pm$01.87		& 92.02$\pm$05.91		& 88.52$\pm$13.91	& 99.41$\pm$01.87		\\
			4	& 98.37$\pm$04.11		& 98.37$\pm$04.11		& 98.49$\pm$04.79		& 95.65$\pm$10.08		& 30.35$\pm$08.46	& 65.14$\pm$10.88		& 98.56$\pm$03.05		& 96.98$\pm$03.88		& 91.22$\pm$10.90	& \bf{98.89$\pm$02.55}	\\
			5	& 97.32$\pm$03.90		& 97.32$\pm$03.90		& \bf{98.05$\pm$04.13}	& 97.08$\pm$03.79		& 29.28$\pm$06.18	& 67.91$\pm$08.45		& 98.30$\pm$02.75		& 95.01$\pm$05.29		& 92.12$\pm$07.16	& \bf{99.03$\pm$01.69}	\\
			6	& 94.15$\pm$04.51		& 94.80$\pm$03.15		& \bf{99.13$\pm$02.74}	& 90.33$\pm$05.49		& 27.82$\pm$06.28	& 64.58$\pm$09.68		& 97.53$\pm$03.78		& 91.76$\pm$05.45		& 87.63$\pm$06.76	& 98.80$\pm$02.15		\\
			7	& 94.87$\pm$04.19		& 94.64$\pm$03.40		& 96.29$\pm$03.44		& 92.08$\pm$04.41		& 29.38$\pm$06.35	& 59.13$\pm$03.78		& 96.60$\pm$04.28		& 87.12$\pm$05.60		& 90.93$\pm$05.83	& \bf{98.40$\pm$01.85}	\\
			8	& 90.91$\pm$05.29		& 93.58$\pm$01.91		& 95.96$\pm$02.93		& 89.95$\pm$03.31		& 29.02$\pm$03.45	& 65.20$\pm$04.20		& 91.20$\pm$03.97		& 89.09$\pm$05.20		& 93.00$\pm$03.96	& \bf{98.52$\pm$01.44}	\\
			9	& 90.86$\pm$03.11		& 91.68$\pm$02.16		& 93.53$\pm$01.75		& 88.37$\pm$03.53		& 31.15$\pm$03.24	& 64.05$\pm$09.52		& 94.06$\pm$03.36		& 89.34$\pm$05.09		& 89.22$\pm$06.95	& \bf{98.34$\pm$01.03}	\\
			10	& 82.68					& 86.07					& 94.16					& 90.20					& 29.75				& 66.82					& 96.50					& 93.54					& 93.38				& \bf{100.0}			\\	\hline
			Average	
			& 93.78					& 94.66					& 97.29					& 93.17					& 26.74				& 69.23					& 96.91					& 92.76					& 91.62				& \bf{99.04}			\\	\hline
		\end{tabular}
	}
	\label{tab_per_jaffe}
\end{table*}

\begin{table*}[!tb] 
	\Huge
	\centering
	\caption{Clustering Performance on PIX }
	\resizebox{0.8\textwidth}{!}{
		\begin{tabular}{ |c| |c |c |c |c |c |c | c| c| c| c| }
			\hline
			\multirow{2}{1cm}{\centering N}& \multicolumn{10}{c|}{Accuracy (\%)}\\	
			\cline{2-11}
			\multirow{2}{1cm}{}
			& K-Means  			& PCA				& RPCA		 		& 2DPCA					& NMF 				& SC				& GNMF 				& RMNMF				& Semi-NMF			& TS-NMF	\\\hline
			2	& 94.50$\pm$10.39	& 94.50$\pm$10.39	& 99.50$\pm$01.58	& 99.50$\pm$01.58		& 73.00$\pm$11.11	& 94.50$\pm$10.39	& 95.50$\pm$08.32	& 96.50$\pm$07.84	& 94.00$\pm$10.22	& \bf{100.0$\pm$00.00}	\\
			3	& 96.00$\pm$05.84	& 96.00$\pm$05.84	& 96.00$\pm$05.84	& 97.67$\pm$06.30		& 60.67$\pm$09.27	& 95.00$\pm$06.89	& 96.00$\pm$05.84	& 97.33$\pm$03.06	& 95.33$\pm$06.13	& \bf{99.00$\pm$01.61}	\\
			4	& 96.25$\pm$04.60	& 97.25$\pm$03.26	& 97.25$\pm$03.81	& 99.25$\pm$01.69		& 59.25$\pm$12.47	& 90.50$\pm$13.58	& 97.50$\pm$03.73	& 96.50$\pm$04.44	& 92.25$\pm$06.17	& \bf{99.75$\pm$00.79}	\\
			5	& 87.20$\pm$11.00	& 93.00$\pm$05.35	& 90.80$\pm$09.34	& 95.40$\pm$08.17		& 55.00$\pm$10.38	& 71.80$\pm$13.62	& 92.80$\pm$06.12	& 90.80$\pm$07.50	& 86.20$\pm$11.09	& \bf{98.80$\pm$01.69}	\\
			6	& 84.83$\pm$12.38	& 88.83$\pm$09.20	& 90.17$\pm$09.51	& 90.00$\pm$11.92		& 47.33$\pm$05.89	& 72.83$\pm$12.74	& 93.17$\pm$04.68	& 89.00$\pm$08.72	& 86.67$\pm$10.60	& \bf{98.83$\pm$01.69}	\\
			7	& 84.14$\pm$05.89	& 90.86$\pm$07.95	& 90.27$\pm$07.64	& 94.86$\pm$05.64		& 51.29$\pm$06.48	& 65.57$\pm$08.31	& 93.29$\pm$06.28	& 87.14$\pm$07.85	& 91.57$\pm$05.85	& \bf{97.57$\pm$02.52}	\\
			8	& 84.12$\pm$05.65	& 85.50$\pm$06.40	& 87.25$\pm$0714	& 95.25$\pm$02.27		& 51.38$\pm$03.79	& 58.00$\pm$06.07	& 84.12$\pm$05.39	& 82.37$\pm$05.38	& 84.50$\pm$06.35	& \bf{97.13$\pm$01.56}	\\
			9	& 83.22$\pm$07.45	& 90.67$\pm$03.32	& 92.89$\pm$00.57	& \bf{96.44$\pm$00.70}	& 45.33$\pm$05.02	& 57.11$\pm$04.03	& 90.56$\pm$06.14	& 87.00$\pm$06.83	& 81.00$\pm$04.81	& 96.11$\pm$01.08		\\
			10	& 80.00				& 89.00				& 80.00				& 87.00					& 11.00				& 62.00				& 92.00				& 81.00				& 81.00				& \bf{98.00}			\\	\hline
			Average	
			& 87.81				& 91.73				& 91.57				& 95.04					& 50.47				& 74.15				& 92.77				& 89.74				& 88.06				& \bf{98.35}			\\	\hline
			\hline
			\multirow{2}{1cm}{\centering N}& \multicolumn{10}{c|}{Normalized Mutual Information (\%)}\\	
			\cline{2-11}
			\multirow{2}{1cm}{}
			& K-Means  			& PCA				& RPCA		 		& 2DPCA				& NMF 				& SC				& GNMF 				& RMNMF				& Semi-NMF			& TS-NMF	\\\hline
			2	& 83.81$\pm$28.77	& 83.81$\pm$28.77	& 97.58$\pm$07.64	& 97.58$\pm$07.64	& 25.16$\pm$21.49	& 83.81$\pm$28.77	& 85.50$\pm$25.35	& 88.28$\pm$22.45	& 81.39$\pm$28.27	& \bf{100.0$\pm$00.00}	\\
			3	& 89.87$\pm$11.64	& 89.87$\pm$11.64	& 89.87$\pm$11.64	& 94.95$\pm$12.80	& 34.04$\pm$13.10	& 88.11$\pm$14.96	& 89.87$\pm$11.64	& 92.32$\pm$08.18	& 88.66$\pm$11.81	& \bf{96.95$\pm$04.91}	\\
			4	& 93.39$\pm$07.13	& 94.65$\pm$05.05	& 94.67$\pm$05.78	& 98.42$\pm$03.44	& 42.33$\pm$16.08	& 86.82$\pm$16.72	& 95.26$\pm$05.31	& 93.45$\pm$07.08	& 87.96$\pm$08.03	& \bf{99.40$\pm$01.91}	\\
			5	& 87.42$\pm$08.35	& 90.76$\pm$06.83	& 90.13$\pm$07.54	& 93.60$\pm$10.71	& 47.52$\pm$10.15	& 67.15$\pm$14.69	& 90.48$\pm$08.38	& 88.04$\pm$07.35	& 85.79$\pm$08.81	& \bf{97.82$\pm$02.94}	\\
			6	& 86.71$\pm$08.47	& 88.51$\pm$06.90	& 90.64$\pm$05.68	& 91.12$\pm$08.83	& 45.19$\pm$07.04	& 69.75$\pm$11.39	& 91.00$\pm$05.85	& 87.05$\pm$07.23	& 86.55$\pm$08.30	& \bf{97.82$\pm$01.51}	\\
			7	& 86.74$\pm$03.63	& 92.16$\pm$04.45	& 91.54$\pm$04.57	& 93.92$\pm$04.67	& 49.54$\pm$08.12	& 64.08$\pm$09.35	& 92.34$\pm$04.69	& 87.06$\pm$07.05	& 90.98$\pm$03.96	& \bf{96.66$\pm$03.11}	\\
			8	& 85.70$\pm$03.87	& 88.28$\pm$02.80	& 90.56$\pm$03.07	& 93.29$\pm$02.72	& 49.33$\pm$03.90	& 60.00$\pm$05.12	& 87.44$\pm$02.03	& 83.54$\pm$04.15	& 86.50$\pm$03.52	& \bf{95.91$\pm$02.20}	\\
			9	& 86.31$\pm$04.83	& 90.06$\pm$02.78	& 92.78$\pm$00.82	& 94.76$\pm$00.94	& 45.26$\pm$03.24	& 60.67$\pm$03.91	& 91.43$\pm$02.90	& 87.89$\pm$04.59	& 85.27$\pm$03.48	& \bf{95.36$\pm$01.18}	\\
			10	& 88.09				& 90.25				& 87.01				& 92.52				& 11.73				& 62.45				& 91.57				& 86.02				& 86.84				& \bf{97.65}			\\	\hline
			Average	
			& 87.56				& 89.82				& 91.64				& 94.46				& 38.90				& 71.43				& 90.54				& 88.18				& 86.66				& \bf{97.51}			\\	\hline
		\end{tabular}
	}
	\label{tab_per_pix10}
\end{table*}

\begin{table*}[!tb] 
	\Huge
	\centering
	\caption{Clustering Performance on Yale }
	\resizebox{0.8\textwidth}{!}{
		\begin{tabular}{ |c| |c |c |c |c |c |c | c| c| c| c| }
			\hline
			\multirow{2}{1cm}{\centering N}& \multicolumn{10}{c|}{Accuracy (\%)}\\	
			\cline{2-11}
			\multirow{2}{1cm}{}
			& K-Means  			& PCA				& RPCA		 			& 2DPCA				& NMF 				& SC				& GNMF 				& RMNMF				& Semi-NMF				& TS-NMF	\\ \hline
			2	& 71.36$\pm$22.78	& 84.09$\pm$17.70	& 73.64$\pm$24.97		& 76.82$\pm$22.62	& 53.64$\pm$04.18	& 71.36$\pm$19.05	& 78.64$\pm$10.06	& 86.36$\pm$13.03	& 58.64$\pm$09.45		& \bf{86.82$\pm$08.42}	\\
			3	& 61.82$\pm$15.20	& 67.58$\pm$15.59	& 63.33$\pm$21.40		& 65.45$\pm$18.75	& 52.12$\pm$04.24	& 52.12$\pm$13.08	& 62.12$\pm$13.50	& 72.12$\pm$12.10	& 67.58$\pm$13.40		& \bf{73.03$\pm$13.51}	\\
			4	& 46.14$\pm$11.54	& 57.95$\pm$10.06	& 49.55$\pm$11.33		& 48.86$\pm$11.10	& 40.45$\pm$07.93	& 44.09$\pm$06.96	& 61.36$\pm$04.91	& 60.45$\pm$10.45	& \bf{65.91$\pm$10.39}	& {65.23$\pm$13.85}	\\
			5	& 50.00$\pm$08.45	& 55.09$\pm$08.13	& 53.27$\pm$09.35		& 54.36$\pm$19.59	& 39.82$\pm$05.65	& 50.73$\pm$16.47	& 57.27$\pm$10.37	& 58.73$\pm$11.37	& 60.18$\pm$09.90		& \bf{64.18$\pm$07.76}	\\
			6	& 50.61$\pm$04.91	& 49.55$\pm$03.78	& 53.18$\pm$03.53		& 52.12$\pm$03.44	& 33.64$\pm$03.90	& 37.58$\pm$07.92	& 49.24$\pm$03.86	& 51.36$\pm$04.07	& 52.73$\pm$03.56		& \bf{55.45$\pm$08.30}	\\
			7	& 47.27$\pm$06.57	& 51.30$\pm$04.55	& 54.55$\pm$04.37		& 51.95$\pm$05.09	& 34.81$\pm$06.12	& 39.61$\pm$11.19	& 50.78$\pm$04.08	& 52.73$\pm$04.02	& 52.99$\pm$04.57		& \bf{59.09$\pm$08.83}	\\
			8	& 48.30$\pm$05.55	& 51.70$\pm$05.57	& 53.98$\pm$04.92		& 52.61$\pm$06.16	& 32.95$\pm$05.59	& 37.73$\pm$10.67	& 51.14$\pm$02.62	& 49.89$\pm$04.68	& 49.89$\pm$07.37		& \bf{57.73$\pm$08.09}	\\
			9	& 45.45$\pm$07.68	& 48.48$\pm$06.46	& \bf{54.44$\pm$02.67}	& 48.28$\pm$05.92	& 30.00$\pm$03.20	& 32.53$\pm$06.28	& 48.69$\pm$04.74	& 50.30$\pm$06.85	& 47.58$\pm$03.91		& 53.33$\pm$07.31		\\
			10	& 43.36$\pm$04.29	& 45.18$\pm$04.04	& 50.00$\pm$04.56		& 46.45$\pm$04.91	& 28.73$\pm$02.36	& 28.91$\pm$05.71	& 46.27$\pm$05.22	& 46.73$\pm$03.81	& 46.64$\pm$06.17		& \bf{51.36$\pm$05.97}	\\
			12	& 40.83$\pm$04.23	& 45.08$\pm$04.79	& \bf{49.92$\pm$05.74}	& 46.44$\pm$03.03	& 27.80$\pm$01.43	& 30.30$\pm$03.50	& 44.55$\pm$06.84	& 45.08$\pm$06.04	& 44.17$\pm$04.72		& \bf{49.92$\pm$02.98}	\\	
			14	& 41.75$\pm$02.82	& 42.92$\pm$03.25	& 45.65$\pm$02.87		& 46.04$\pm$02.09	& 25.06$\pm$00.98	& 25.52$\pm$03.02	& 42.79$\pm$03.89	& 44.61$\pm$03.33	& 42.92$\pm$04.02		& \bf{49.81$\pm$04.14}	\\	
			15	& 39.39				& 41.21				& 44.24					& 42.42				& 21.82				& 21.82				& 43.03				& 44.85				& 36.36					& \bf{53.33}			\\	\hline
			Average		
			& 48.88				& 53.34				& 53.81					& 52.65				& 35.07				& 39.36				& 53.00				& 55.27				& 52.13					& \bf{59.94}			\\	\hline
			\hline
			\multirow{2}{1cm}{\centering N}& \multicolumn{10}{c|}{Normalized Mutual Information (\%)}\\	
			\cline{2-11}
			\multirow{2}{1cm}{}
			& K-Means  			& PCA				& RPCA		 			& 2DPCA				& NMF 				& SC				& GNMF 				& RMNMF					&	Semi-NMF			& TS-NMF	\\ \hline
			2	& 35.14$\pm$42.14	& 53.01$\pm$34.53	& 43.23$\pm$46.31		& 44.37$\pm$41.38	& 01.15$\pm$01.90	& 30.43$\pm$29.47	& 33.70$\pm$17.53	& \bf{52.87$\pm$29.87}	& 05.03$\pm$06.37		& 51.94$\pm$23.05	\\
			3	& 33.94$\pm$22.01	& 45.06$\pm$23.93	& 43.87$\pm$31.44		& 41.45$\pm$28.65	& 18.91$\pm$10.05	& 25.24$\pm$20.95	& 40.60$\pm$22.19	& 46.94$\pm$18.56		& 47.76$\pm$17.66		& \bf{48.77$\pm$23.49}	\\
			4	& 23.75$\pm$16.61	& 37.70$\pm$14.89	& 33.70$\pm$15.98		& 27.19$\pm$15.06	& 14.95$\pm$10.54	& 21.25$\pm$11.23	& 42.22$\pm$09.13	& 39.72$\pm$12.17		& \bf{48.23$\pm$13.68}	& 47.58$\pm$15.73	\\
			5	& 35.68$\pm$14.72	& 43.57$\pm$12.14	& 44.13$\pm$12.21		& 39.87$\pm$26.15	& 23.49$\pm$09.30	& 33.17$\pm$22.40	& 43.03$\pm$13.06	& 44.78$\pm$15.32		& 46.26$\pm$15.40		& \bf{49.17$\pm$09.44}	\\
			6	& 40.74$\pm$07.93	& 38.17$\pm$07.05	& 44.67$\pm$05.95		& 42.32$\pm$07.11	& 20.54$\pm$07.17	& 21.77$\pm$12.82	& 37.43$\pm$04.82	& 38.37$\pm$06.02		& 40.97$\pm$04.98		& \bf{45.74$\pm$10.66}	\\
			7	& 39.44$\pm$07.78	& 44.18$\pm$04.96	& 48.63$\pm$06.08		& 46.31$\pm$05.75	& 27.41$\pm$08.44	& 29.02$\pm$16.15	& 42.54$\pm$05.36	& 43.89$\pm$06.09		& 45.57$\pm$03.90		& \bf{52.49$\pm$11.04}	\\
			8	& 43.51$\pm$05.72	& 46.22$\pm$07.16	& 49.65$\pm$06.68		& 47.97$\pm$08.00	& 26.32$\pm$06.87	& 33.19$\pm$14.25	& 45.36$\pm$04.73	& 44.85$\pm$05.35		& 45.41$\pm$09.02		& \bf{54.19$\pm$09.36}	\\
			9	& 42.94$\pm$08.06	& 45.36$\pm$05.35	& \bf{51.76$\pm$03.55}	& 46.80$\pm$06.43	& 26.60$\pm$05.05	& 25.19$\pm$08.62	& 44.46$\pm$04.74	& 45.97$\pm$06.09		& 44.40$\pm$04.24		& 50.64$\pm$08.36		\\
			10	& 43.73$\pm$03.17	& 45.12$\pm$03.42	& {49.02$\pm$05.39}		& 46.39$\pm$05.34	& 27.50$\pm$03.24	& 23.94$\pm$06.23	& 43.89$\pm$04.62	& 44.59$\pm$04.10		& 45.39$\pm$05.70		& \bf{50.39$\pm$06.97}	\\
			12	& 42.83$\pm$04.26	& 46.61$\pm$03.63	& {51.76$\pm$02.82}		& 50.24$\pm$02.37	& 30.29$\pm$01.43	& 31.53$\pm$05.45	& 46.40$\pm$05.74	& 46.56$\pm$04.94		& 45.69$\pm$04.26		& \bf{53.50$\pm$03.58}	\\
			14	& 46.20$\pm$02.59	& 47.46$\pm$02.20	& {50.86$\pm$02.13}		& 51.36$\pm$01.95	& 30.63$\pm$00.76	& 28.78$\pm$04.08	& 46.50$\pm$03.59	& 47.77$\pm$02.04		& 46.18$\pm$02.98		& \bf{53.06$\pm$03.16}	\\
			15	& 43.84				& 47.86				& 49.64					& 49.36				& 29.25				& 26.86				& 48.34				& 48.20					& 40.79					& \bf{55.21}			\\	\hline
			Average		
			& 39.31				& 45.03				& 46.74					& 44.47				& 23.09				& 27.53				& 42.87				& 45.38					& 41.81					& \bf{51.06}			\\	\hline
		\end{tabular}
	}
	\label{tab_per_yaleb15}
\end{table*}

\subsection{Algorithms in Comparison}
To illustrate the effectiveness of TS-NMF, we compare it with 9 methods as baselines. We summarize the methods and settings as follows: 1) \textbf{K-means}. It is one of the most widely used clustering algorithm, where a fast implementation\footnote{\scriptsize\url{http://www.cad.zju.edu.cn/home/dengcai/Data/Clustering.html}} is used in the experiment. It is also used as a final step of other methods. 2) \textbf{PCA} \cite{jolliffe2002principal}. We seek $r$ principal components with the most variations, with $r$ chosen from $\{1,3,5,7,9\}$. 3) \textbf{Robust PCA} (RPCA) \cite{candes2011robust}. Particularly, we use the inexact augmented Lagrange multiplier (IALM) method \cite{lin2010augmented} for its optimization, where the theoretically optimal parameter provided in \cite{candes2011robust} is used. 4) \textbf{2DPCA} \cite{yang2004two}. We choose different $r$ values in $\{1,3,5,7,9\}$ as the number of projection directions. 5) \textbf{SC} \cite{ng2002spectral}. RBF kernel is used to construct the graph Laplacian with radius chosen from $\mathcal{S}=\{10^{-3},10^{-2},10^{-1},10^{0},10^{1},10^{2},10^{3}\}$. 6) \textbf{NMF} \cite{lee1999learning}. It is closely related with K-means and SC. The standard multiplicative update rules are used. 7) \textbf{GNMF} \cite{cai2011graph}. Frobenius-norm based loss function is used. 8) \textbf{RMNMF} \cite{huang2014robust}. As suggested by the original paper, we fix $\rho$ and $\mu$ to be 1E-5 and 1.1 in ALM framework for its optimization. 9) \textbf{Semi-NMF} \cite{ding2010convex}. It is more general than NMF by relaxing the basis with mixed signs. 10) \textbf{TS-NMF}. We choose $r$ from $\{1,3,5,7,9\}$. Together with GNMF and RMNMF, the regularization parameter is chosen from $\mathcal{S}$, and the graph Laplacian is constructed using binary weighting with 5 neighbors.

\subsection{Clustering Performance}
In this subsection, we present the clustering performance in detail. For a given data set, the total number of clusters is $\bar{N}$. For more detailed comparison, we randomly select $N\le\bar{N}$ clusters as a subset to conduct experiments. Different data sets may have different sets of $N$ values, thus we present the detailed information in the first column of \cref{tab_per_eyaleb,tab_per_jaffe,tab_per_orl,tab_per_pix10,tab_per_semeion,tab_per_yaleb15}. For a specific data set and algorithm, we conduct experiments on 10 randomly selected subsets for each $N$ value with all combinations of parameters tested. Then we report the average performance as well as the standard devision for each $N$ value. The clustering results are presented in \cref{tab_per_eyaleb,tab_per_jaffe,tab_per_orl,tab_per_pix10,tab_per_semeion,tab_per_yaleb15}.

\begin{figure*}[htbp]
\centering

\includegraphics[width=2\columnwidth]{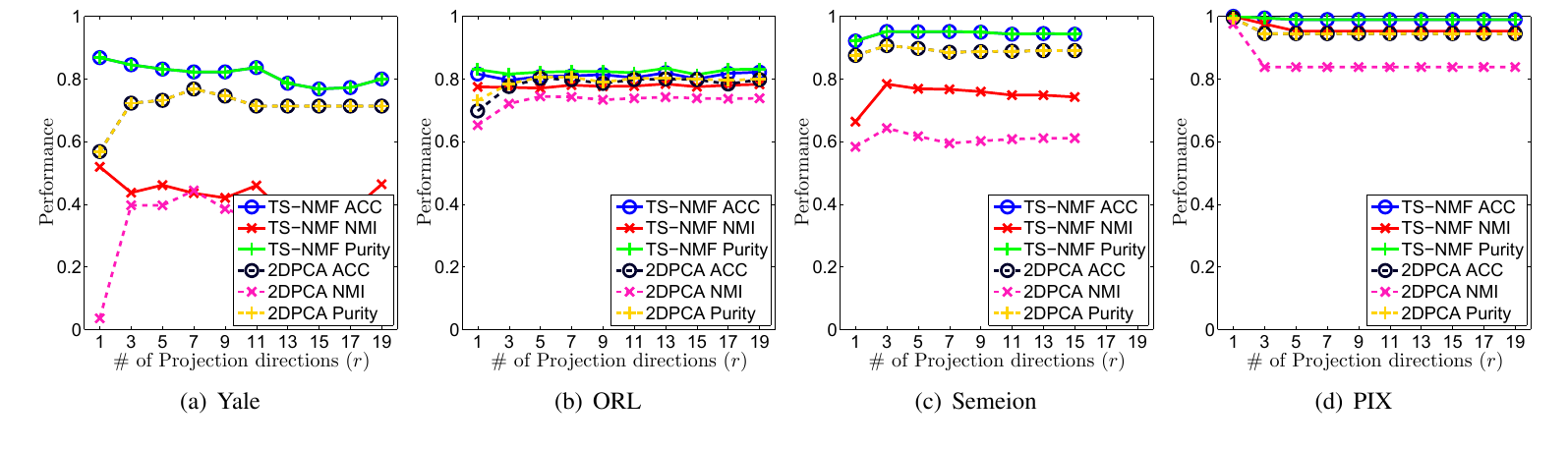}

\caption{ Performance variations in accuracy, NMI, and purity with respect to different $r$ values on Yale, ORL, Semeion, and PIX data sets, respectively. }
\label{fig_r}
\end{figure*}

It is seen that the proposed TS-NMF significantly outperforms the other methods. For example, on EYaleB data, TS-NMF improves the performance from GNMF, the second best, by around 40\% in accuracy, NMI, and purity. It is also noted that EYaleB and ORL data are highly noisy and afflicted with outliers due to the shadows and wearings on faces, hence the results suggest TS-NMF can better deal with noisy data. We also observe good, yet less competitive, results from other methods. For example, on JAFFE data, RPCA has good performance, but its performance degrades more significantly than TS-NMF when $N$ increases, hence TS-NMF shows better stability with large $N$ values and is more suitable for real world applications. Also, it is observed that different methods have achieved the second best performance on different data sets, whereas TS-NMF always has the best. This observation implies that the other methods may have good performance on some data sets, but may do poorly on the others. The stability of TS-NMF on different data sets suggests its high potential for real world applications.

\subsection{Clustering Performance on Corrupted Data}
As mentioned in previous subsection, TS-NMF has shown its effectiveness on noisy data. To better illustrate this, in this test, we compare TS-NMF with other methods on corrupted ORL data, with randomly 40\%, and 60\% entries removed, respectively. The experimental setting is the same as before. From \cref{tab_per_orl_r40,tab_per_orl_r60}, it is observed that TS-NMF has the best performance with significant improvements, which, again, confirms the robustness of TS-NMF.

\begin{table*}[!tb] 
	\normalsize
	\centering
	\caption{ Clustering Performance on 40\% Corrupted ORL }
	\resizebox{0.8\textwidth}{!}{
		\begin{tabular}{ |c| |c |c |c | c |c |c | c| c| c| c| }
			\hline
			\multirow{2}{1cm}{\centering N}& \multicolumn{10}{c|}{Accuracy (\%)}\\	
			\cline{2-11}
			\multirow{2}{1cm}{}
			& K-Means  			& PCA		 		& RPCA	 			& 2DPCA				& SC				& NMF 				& GNMF 				& RMNMF				& Semi-NMF			& TS-NMF				\\ \hline
			2	& 78.50$\pm$15.82	& 77.00$\pm$15.49	& 65.50$\pm$09.85 	& 77.50$\pm$20.03	& ---------------	& 61.00$\pm$06.58	& 85.50$\pm$14.99	& 76.50$\pm$16.67	& 87.00$\pm$15.49	& \bf{90.00$\pm$15.46}	\\
			4	& 49.75$\pm$09.68	& 37.50$\pm$04.71	& 49.50$\pm$06.65	& 59.00$\pm$10.08	& ---------------	& 38.75$\pm$04.45	& 62.25$\pm$07.31	& 53.50$\pm$10.01	& 56.50$\pm$07.75	& \bf{67.00$\pm$10.92}	\\
			6	& 36.33$\pm$07.89	& 31.83$\pm$04.87	& 43.67$\pm$07.28	& 48.17$\pm$12.36	& ---------------	& 31.67$\pm$02.48	& 53.33$\pm$08.82	& 45.50$\pm$10.45	& 52.83$\pm$11.79	& \bf{59.00$\pm$10.22}	\\
			8	& 30.38$\pm$04.72	& 27.62$\pm$03.75	& 41.25$\pm$05.56	& 42.75$\pm$06.42	& ---------------	& 29.75$\pm$04.16	& 42.63$\pm$06.22	& 34.25$\pm$04.87	& 41.25$\pm$05.77	& \bf{51.50$\pm$07.40}	\\
			10	& 29.27$\pm$03.58	& 27.10$\pm$03.96	& 34.50$\pm$04.22	& 39.40$\pm$05.13	& ---------------	& 11.00$\pm$00.00	& 39.30$\pm$05.68	& 33.10$\pm$03.60	& 34.90$\pm$05.45	& \bf{48.30$\pm$06.27}	\\
			12	& 27.42$\pm$02.40	& 24.17$\pm$01.52	& 33.50$\pm$04.02	& 34.50$\pm$04.93	& ---------------	& 23.33$\pm$01.80	& 35.58$\pm$02.91	& 30.75$\pm$02.65	& 33.17$\pm$02.63	& \bf{45.58$\pm$05.47}	\\
			14	& 25.50$\pm$02.78	& 24.07$\pm$03.11	& 32.64$\pm$03.63	& 36.21$\pm$04.73	& ---------------	& 23.29$\pm$01.31	& 34.50$\pm$05.60	& 28.50$\pm$02.98	& 33.50$\pm$03.83	& \bf{45.07$\pm$04.70}	\\
			16	& 23.81$\pm$02.80	& 22.25$\pm$01.82	& 29.25$\pm$02.55	& 33.88$\pm$03.06	& ---------------	& 21.69$\pm$00.84	& 30.50$\pm$03.18	& 25.94$\pm$01.57	& 30.75$\pm$03.26	& \bf{40.00$\pm$04.81}	\\
			18	& 23.06$\pm$02.03	& 21.72$\pm$01.90	& 30.72$\pm$03.74	& 32.06$\pm$02.80	& ---------------	& 21.39$\pm$01.29	& 32.06$\pm$03.44	& 25.83$\pm$01.58	& 28.83$\pm$00.92	& \bf{36.56$\pm$03.51}	\\
			20	& 22.10$\pm$02.00	& 20.55$\pm$01.30	& 27.70$\pm$03.10	& 31.95$\pm$03.07	& ---------------	& 19.85$\pm$01.36	& 31.25$\pm$03.55	& 26.35$\pm$01.43	& 28.05$\pm$01.46	& \bf{37.35$\pm$03.01}	\\	\hline
			Average	
			& 34.61				& 31.38				& 35.57				& 43.54				& ---------------	& 25.84				& 44.79				& 38.02				& 42.68				& \bf{52.04}	\\	\hline
			\hline
			\multirow{2}{1cm}{\centering N}& \multicolumn{10}{c|}{Normalized Mutual Information (\%)}\\	
			\cline{2-11}
			\multirow{2}{1cm}{}
			& K-Means  			& PCA		 		& RPCA	 			& 2DPCA				& SC				& NMF 				& GNMF 				& RMNMF				& Semi-NMF			& TS-NMF				\\ \hline
			2	& 43.76$\pm$35.61	& 39.35$\pm$26.76	& 12.74$\pm$12.09	& 43.50$\pm$44.67	& ---------------	& 07.33$\pm$09.71	& 55.76$\pm$37.73	& 34.46$\pm$33.07	& 61.98$\pm$42.68	& \bf{72.07$\pm$39.04}	\\
			4	& 26.82$\pm$12.41	& 11.52$\pm$04.32	& 24.67$\pm$06.12	& 38.37$\pm$12.68	& ---------------	& 11.68$\pm$04.82	& 44.35$\pm$11.61	& 32.00$\pm$12.80	& 42.35$\pm$11.29	& \bf{59.04$\pm$15.12}	\\
			6	& 20.64$\pm$09.85	& 16.11$\pm$05.78	& 30.53$\pm$09.59	& 35.82$\pm$13.95	& ---------------	& 15.68$\pm$03.16	& 44.53$\pm$13.59	& 31.24$\pm$10.86	& 43.31$\pm$15.06	& \bf{54.98$\pm$14.69}	\\
			8	& 22.45$\pm$06.22	& 18.02$\pm$04.77	& 36.09$\pm$06.28	& 36.88$\pm$07.35	& ---------------	& 20.69$\pm$03.54	& 38.87$\pm$08.70	& 26.83$\pm$06.68	& 34.57$\pm$05.81	& \bf{47.83$\pm$09.88}	\\
			10	& 25.37$\pm$04.16	& 20.30$\pm$03.89	& 33.73$\pm$04.38	& 38.11$\pm$05.97	& ---------------	& 11.73$\pm$00.00	& 39.79$\pm$05.90	& 29.77$\pm$04.75	& 35.12$\pm$05.17	& \bf{51.47$\pm$06.31}	\\
			12	& 26.16$\pm$03.84	& 21.50$\pm$02.02	& 35.80$\pm$03.90	& 36.13$\pm$06.41	& ---------------	& 24.13$\pm$01.96	& 39.77$\pm$02.93	& 30.12$\pm$02.11	& 36.74$\pm$02.63	& \bf{50.76$\pm$07.10}	\\
			14	& 26.10$\pm$02.79	& 23.45$\pm$03.20	& 38.26$\pm$03.47	& 41.14$\pm$05.66	& ---------------	& 27.07$\pm$02.20	& 40.85$\pm$04.35	& 30.89$\pm$03.02	& 39.70$\pm$03.34	& \bf{52.71$\pm$05.10}	\\
			16	& 27.18$\pm$02.80	& 25.38$\pm$02.34	& 35.31$\pm$02.43	& 41.68$\pm$03.10	& ---------------	& 27.75$\pm$00.68	& 39.21$\pm$03.88	& 31.74$\pm$02.58	& 38.21$\pm$03.12	& \bf{50.33$\pm$04.77}	\\
			18	& 27.78$\pm$01.76	& 25.08$\pm$01.77	& 38.96$\pm$03.73	& 42.12$\pm$02.52	& ---------------	& 29.92$\pm$01.17	& 41.48$\pm$02.99	& 33.37$\pm$01.92	& 38.62$\pm$01.02	& \bf{49.06$\pm$03.72}	\\
			20	& 28.41$\pm$02.30	& 26.77$\pm$02.10	& 37.79$\pm$03.30	& 43.96$\pm$03.97	& ---------------	& 30.09$\pm$01.33	& 43.16$\pm$02.50	& 36.87$\pm$01.33	& 39.21$\pm$01.79	& \bf{50.44$\pm$03.06}	\\	\hline
			Average	
			& 27.47				& 22.75				& 32.39				& 39.77				& ---------------	& 20.61				& 42.78				& 31.74				& 40.98				& \bf{53.87}	\\	\hline
		\end{tabular}
	}
	\label{tab_per_orl_r40}
\end{table*}

\begin{table*}[!tb]
	\normalsize
	\centering
	\caption{ Clustering Performance on 60\% Corrupted ORL }
	\resizebox{0.8\textwidth}{!}{
		\begin{tabular}{ |c| |c |c |c | c |c |c | c| c| c| c| }
			\hline
			\multirow{2}{1cm}{\centering N}& \multicolumn{10}{c|}{Accuracy (\%)}\\	
			\cline{2-11}
			\multirow{2}{1cm}{}
			& K-Means  			& PCA		 		& RPCA	 			& 2DPCA				& SC				& NMF 				& GNMF 				& RMNMF				& Semi-NMF			& TS-NMF				\\ \hline
			2	& 72.50$\pm$13.18	& 59.00$\pm$06.58	& 55.00$\pm$00.00	& 67.00$\pm$12.74	& ---------------	& 59.00$\pm$06.58	& 74.50$\pm$14.99	& 66.50$\pm$11.56	& 73.00$\pm$13.37	& \bf{82.50$\pm$14.77}	\\
			4	& 42.25$\pm$05.83	& 39.50$\pm$05.24	& 27.50$\pm$00.00	& 44.00$\pm$08.51	& ---------------	& 38.00$\pm$03.07	& 47.25$\pm$06.29	& 44.75$\pm$06.06	& 48.50$\pm$08.99	& \bf{55.50$\pm$12.74}	\\
			6	& 33.00$\pm$04.50	& 29.17$\pm$01.42	& 18.33$\pm$00.00	& 37.00$\pm$09.26	& ---------------	& 31.50$\pm$03.80	& 39.67$\pm$07.36	& 35.50$\pm$04.91	& 39.83$\pm$08.80	& \bf{45.17$\pm$05.64}	\\
			8	& 29.25$\pm$01.79	& 25.87$\pm$03.73	& 13.75$\pm$00.00	& 32.25$\pm$03.27	& ---------------	& 28.50$\pm$02.75	& 32.75$\pm$03.28	& 29.13$\pm$02.95	& 35.63$\pm$04.05	& \bf{39.63$\pm$05.53}	\\
			10	& 26.20$\pm$01.69	& 24.70$\pm$02.41	& 11.00$\pm$00.00	& 28.30$\pm$03.56	& ---------------	& 11.00$\pm$00.00	& 29.80$\pm$01.81	& 26.30$\pm$02.36	& 30.60$\pm$03.13	& \bf{35.90$\pm$03.73}	\\
			12	& 23.92$\pm$02.39	& 22.75$\pm$02.72	& 10.00$\pm$00.00	& 27.17$\pm$02.49	& ---------------	& 24.33$\pm$01.88	& 28.33$\pm$03.26	& 24.33$\pm$02.63	& 25.92$\pm$01.59	& \bf{32.67$\pm$03.57}	\\
			14	& 23.36$\pm$01.69	& 22.00$\pm$01.38	& 08.57$\pm$00.00	& 25.21$\pm$02.36	& ---------------	& 22.14$\pm$00.89	& 27.21$\pm$02.24	& 24.14$\pm$02.23	& 26.07$\pm$02.70	& \bf{32.29$\pm$05.00}	\\
			16	& 22.00$\pm$01.21	& 20.50$\pm$01.84	& 07.50$\pm$00.00	& 24.69$\pm$02.29	& ---------------	& 21.25$\pm$01.41	& 25.62$\pm$01.11	& 21.94$\pm$01.65	& 25.19$\pm$02.02	& \bf{30.69$\pm$02.83}	\\
			18	& 20.94$\pm$01.17	& 20.11$\pm$01.41	& 06.67$\pm$06.00	& 24.22$\pm$01.05	& ---------------	& 21.06$\pm$01.21	& 23.89$\pm$01.92	& 21.33$\pm$01.18	& 23.99$\pm$02.10	& \bf{27.89$\pm$02.98}	\\
			20	& 21.10$\pm$01.76	& 19.45$\pm$01.26	& 06.00$\pm$00.00	& 23.65$\pm$02.59	& ---------------	& 20.75$\pm$01.32	& 23.40$\pm$00.94	& 20.85$\pm$01.93	& 22.30$\pm$01.34	& \bf{27.53$\pm$01.53}	\\	\hline
			Average	
			& 31.45				& 28.31				& 16.43				& 33.35				& ---------------	& 27.75				& 35.24				& 31.48				& 35.04				& \bf{40.96}	\\	\hline
			\hline
			\multirow{2}{1cm}{\centering N}& \multicolumn{10}{c|}{Normalized Mutual Information (\%)}\\	
			\cline{2-11}
			\multirow{2}{1cm}{}
			& K-Means  			& PCA		 		& RPCA	 			& 2DPCA				& SC				& NMF 				& GNMF 				& RMNMF				& Semi-NMF			& TS-NMF				\\ \hline
			2	& 26.61$\pm$23.96	& 06.01$\pm$05.74	& 05.19$\pm$00.00	& 17.01$\pm$16.16	& ---------------	& 03.86$\pm$04.01	& 28.85$\pm$25.19	& 14.33$\pm$15.97	& 23.62$\pm$25.78	& \bf{48.49$\pm$33.15}	\\
			4	& 15.79$\pm$08.02	& 11.75$\pm$05.72	& 08.20$\pm$00.00	& 15.80$\pm$08.36	& ---------------	& 07.87$\pm$02.82	& 22.13$\pm$08.67	& 16.33$\pm$08.77	& 24.90$\pm$12.69	& \bf{32.05$\pm$15.32}	\\
			6	& 16.06$\pm$04.27	& 12.74$\pm$01.99	& 09.56$\pm$00.00	& 23.44$\pm$11.20	& ---------------	& 14.05$\pm$03.90	& 25.29$\pm$08.36	& 20.78$\pm$05.58	& 26.91$\pm$09.36	& \bf{33.27$\pm$10.22}	\\
			8	& 20.66$\pm$02.71	& 15.91$\pm$03.60	& 10.60$\pm$00.00	& 23.65$\pm$04.99	& ---------------	& 18.06$\pm$02.48	& 24.63$\pm$03.52	& 18.89$\pm$02.97	& 29.10$\pm$04.98	& \bf{35.03$\pm$07.00}	\\
			10	& 21.07$\pm$02.14	& 18.84$\pm$03.00	& 11.73$\pm$00.00	& 23.88$\pm$04.12	& ---------------	& 11.73$\pm$00.00	& 28.19$\pm$02.44	& 20.44$\pm$03.48	& 28.58$\pm$03.59	& \bf{35.28$\pm$02.54}	\\
			12	& 23.22$\pm$03.14	& 20.43$\pm$03.23	& 12.36$\pm$00.00	& 28.30$\pm$02.50	& ---------------	& 24.64$\pm$02.97	& 30.78$\pm$03.00	& 22.14$\pm$04.06	& 27.77$\pm$01.37	& \bf{35.14$\pm$03.31}	\\
			14	& 25.32$\pm$02.41	& 22.44$\pm$01.69	& 11.96$\pm$00.00	& 29.14$\pm$02.51	& ---------------	& 25.36$\pm$01.62	& 32.34$\pm$02.79	& 24.18$\pm$04.22	& 30.81$\pm$03.14	& \bf{38.18$\pm$05.50}	\\
			16	& 26.09$\pm$01.70	& 23.75$\pm$02.03	& 11.81$\pm$00.00	& 30.63$\pm$02.90	& ---------------	& 26.75$\pm$02.27	& 33.00$\pm$01.88	& 25.45$\pm$01.67	& 32.11$\pm$01.93	& \bf{37.55$\pm$02.41}	\\
			18	& 27.10$\pm$01.42	& 24.88$\pm$01.73	& 11.86$\pm$00.00	& 31.93$\pm$01.43	& ---------------	& 29.55$\pm$01.23	& 33.56$\pm$01.54	& 27.93$\pm$01.71	& 32.38$\pm$02.07	& \bf{37.77$\pm$02.46}	\\
			20	& 28.84$\pm$02.03	& 26.22$\pm$01.29	& 12.13$\pm$00.00	& 33.40$\pm$03.26	& ---------------	& 31.41$\pm$01.48	& 34.70$\pm$01.17	& 28.52$\pm$03.42	& 33.63$\pm$02.15	& \bf{38.96$\pm$01.56}	\\	\hline
			Average	
			& 23.08				& 18.30				& 10.54				& 25.72				& ---------------	& 19.33				& 29.35				& 21.90				& 29.25				& \bf{37.17}	\\	\hline
		\end{tabular}
	}
	\label{tab_per_orl_r60}
\end{table*}

\subsection{Parameter Sensitivity}
In this subsection, we show the effects of the parameters on clustering performance. Due to the space limit, we present the results on part of the data sets. We first compare TS-NMF and 2DPCA by showing their performance variation with respect to $r$. For better illustration, a wider range of values, i.e., $\{1,3,5,\cdots,19\}$, is considered for $r$. We report the results in \cref{fig_exp_per}. For TS-NMF, we tune $\{\lambda_1,\lambda_2\}\in\mathcal{S}\times\mathcal{S}$, such that the best performance is obtained for each fixed $r$ value. It is observed that the proposed method outperforms 2DPCA when the same value of $r$ is used. Also, it is observed that with only a small number of projection directions, TS-NMF can achieve very good performance, which confirms the key idea of this paper and significantly reduces the cost of computations for solving $P$ and $Q$.

Then we test the performance of TS-NMF with respect to different combinations of $\lambda_1$ and $\lambda_2$. We report the performance with $r$ tuned such that the best results are observed. From \cref{fig_exp_per}, it is seen that the performance of TS-NMF is high in a wide range of parameter combinations. Here, we only show the results of ORL and Semeion data, but similar patterns can be observed and similar conclusion can be drawn from other data sets. These observations show insensitivities of TS-NMF to parameters, which indicates its ease of use in real world applications.

%
%

\begin{figure*}[!tb]
\centering

\includegraphics[width=2\columnwidth]{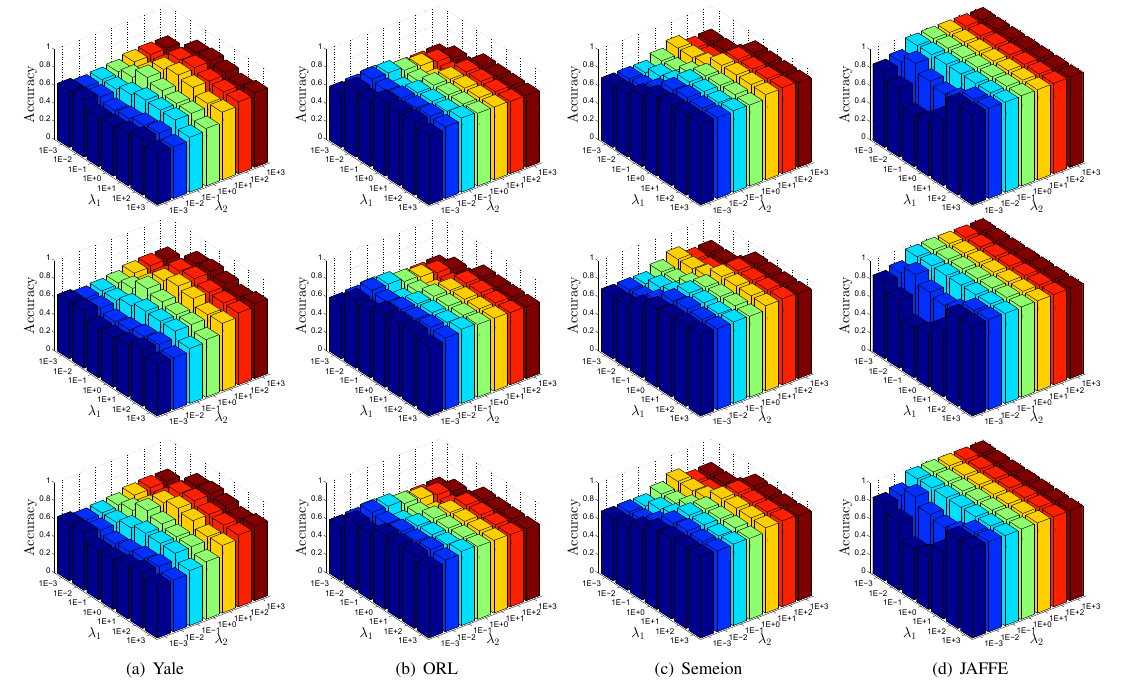}

%
%
%
\caption{ Performance variations in accuracy, NMI, and purity with respect to different combinations of $\lambda_1$ and $\lambda_2$ values on yale, ORL, Semeion, and JAFFE data sets. }
\label{fig_exp_per}
\end{figure*}

%
%

\section{Conclusion}
\label{sec_conclusion}
In this paper, we propose a 2-dimensional semi-nonnegative matrix factorization (TS-NMF) model for clustering. For 2D data, existing methods usually convert the examples to vectors, which fails to fully exploit the spatial information from the data. The proposed method overcomes this limitation by retaining the spatial information of the data for effective data representation. It seeks optimal projection directions under the guidance of new data representations and the goal of clustering. With these projections, a subspace where 2D data are projected can be found with the most expressive information. Moreover, the projected data are used to construct a manifold, which is adaptively updated according to the projections and less afflicted by noise. Hence, three tasks of seeking optimal projection directions, recovering new data representations, and manifold learning mutually enhance each other, rendering TS-NMF a powerful tool for representing 2D data. Due to the use of major information from the data, TS-NMF is robust to noise, occlusions or corruptions. We have performed extensive experiments and the results have confirmed the effectiveness of TS-NMF. The promising clustering performance and insensitivity to parameters suggest high potential of TS-NMF for real world applications. While TS-NMF is proposed to resolve the problem of omitting spatial information with existing methods when processing 2D data, it is also suitable for 1-dimensional data by treating vectors as special cases of 2D matrices. Hence, TS-NMF can be applied to general data with minor modifications.


\section*{Acknowledgment}
Qiang Cheng is the corresponding author. This work is supported by the National Science Foundation, under grant IIS-1218712; and the National Natural Science Foundation of China, under grant 11241005. 

\appendix 
\section{Appendix}

Before we prove \thmref{thm_decrease_V}, we first give some useful definition, propositions, and lemmas. 
\begin{definition}
$J(V,V')$ is an auxiliary function for $F(V)$ if $J(V,V')\ge F(V)$ and $J(V,V)=F(V)$.
\end{definition}

\begin{prop}
\label{prop_aux}
Define the following updating procedure,
\begin{equation}
\label{eq_aux_upadte}
V^{(t+1)} = \argmin_{V} J(V,V^{(t)}),
\end{equation}
then we can obtain the following chain of inequalities
\begin{equation}
F(V^{(t+1)}) \!\le\! J(V^{(t+1)},V^{(t)}) \!\le\! J(V^{(t)},V^{(t)}) \!=\! F(V^{(t)}),
\end{equation}
where $t$ denotes the iteration number. Hence, $\{F(V^{(t)})\}$ is decreasing (non-increasing) with \cref{eq_aux_upadte}.
\end{prop}

\begin{prop}[\cite{ding2010convex}]
\label{prop_inequality}
For any matrices $\Theta\in\mathcal{R}_{+}^{n\times n}$, $\Omega\in\mathcal{R}_{+}^{k\times k}$, $S\in\mathcal{R}_{+}^{n\times k}$, and $S'\in\mathcal{R}_{+}^{n\times k}$, with $\Theta$ and $\Omega$ being symmetric, the following inequality holds:
\begin{equation}
\label{eq_ineq_prop}
\sum_{i=1}^{n}\sum_{s=1}^{k}\frac{(\Theta S' \Omega)_{is} S_{is}^2}{S'_{is}} \ge \tr{S^T \Theta S \Omega}.
\end{equation}
\end{prop}

\begin{lemma}
\label{thm_aux}
For F(V) in \cref{eq_dec_proof_obj}, the following function,
%
%
\begin{equation}
\begin{aligned}
	&	J(V,V') \\
	= & - 2\sum_{ik}(B_1^+)_{ik}v'_{ik}(1+\log \frac{v_{ik}}{v'_{ik}})	\\
	& + 2\sum_{ik}(B_1^-)_{ik}\frac{v_{ik}^2+{V'}_{ik}^2}{2{V'}_{ik}^2}	
	 + \sum_{ik}\frac{{(V'A_1^+)}_{ik}v_{ik}^2}{v'_{ik}}	\\
	& - \sum_{ikl}(A_1^-)_{kl}v'_{ik}v'_{il}(1+\log\frac{v_{ik}v_{il}}{v'_{ik}v'_{il}})	\\
	& - 2\sum_{ik}(B_2^+)_{ik}v'_{ik}(1+\log \frac{v_{ik}}{v'_{ik}}) 	\\
	& + 2\sum_{ik}(B_2^-)_{ik}\frac{v_{ik}^2+{V'}_{ik}^2}{2{V'}_{ik}^2} + \sum_{ik}\frac{{(V'A_2^+)}_{ik}v_{ik}^2}{v'_{ik}} \\
	& - \sum_{ikl}(A_2^-)_{kl}v'_{ik}v'_{il}(1+\log\frac{v_{ik}v_{il}}{v'_{ik}v'_{il}})	 \\
	& + \lambda_2 \sum_{ik} \frac{(D_PV')_{ik}v_{ik}^2}{v'_{ik}} \\
	& - \lambda_2 \sum_{ikl} (W_P)_{kl}v'_{ki}v'_{li}(1+\log\frac{v_{ki}v_{li}}{v'_{ki}v'_{li}}) \\
	\end{aligned}\end{equation}\begin{equation}\begin{aligned}\nonumber
	& + \lambda_2 \sum_{ik} \frac{(D_QV')_{ik}v_{ik}^2}{v'_{ik}} \\
	& - \lambda_2 \sum_{ikl} (W_Q)_{kl}v'_{ki}v'_{li}(1+\log\frac{v_{ki}v_{li}}{v'_{ki}v'_{li}}),
\end{aligned}
\end{equation}
is an auxiliary function. Furthermore, $J(V,V')$ is convex in $V$, and its global minimum is
%
\begin{equation}
\label{eq_update_v_proof}
v_{ij} = v'_{ij}
\sqrt{
\frac{(B_1^{+} + B_2^{+} + V'(A_1^{-} + A_2^{-}) + \lambda_2 (W_P + W_Q) V')_{ij}}
{(B_1^{-} + B_2^{-} + V'(A_1^{+} + A_2^{+}) + \lambda_2 (D_P + D_Q) V')_{ij}}
}.
\end{equation}
\end{lemma}
\begin{proof}
We first prove the first statement. For \cref{eq_dec_proof_obj}, we first find an upper bound for each positive term in the following. According to \cref{prop_inequality}, with $\Theta \leftarrow I$ and $\Omega \leftarrow A_\kappa^+$ for $\kappa$=1 or 2, we obtain upper bounds for the 3rd and 7th terms
\begin{equation}
\label{eq_bound_3_7}
\tr{V A_\kappa^+ V^T} \le \sum_{ik}\frac{(V'A_\kappa^+)_{ik}v_{ik}^2}{{V'}_{ik}}.
\end{equation}
By the inequality $\alpha \le \frac{\alpha^2+\beta^2}{2\beta}$ for $\alpha, \beta > 0$, we get the following upper bounds for the 2nd and 6th terms
\begin{equation}
\label{eq_bound_2_6}
\tr{V'B_\kappa^-} = \sum_{ik}v_{ik}(B_\kappa^-)_{ik} \le \sum_{ik}(B_\kappa^-)_{ik}\frac{v_{ik}^2 + {V'}_{ik}^2}{2v'_{ik}}.	
\end{equation}
According to \cref{prop_inequality}, with $\Theta \leftarrow D_\kappa$ for $\kappa$=1 or 2, and $\Omega \leftarrow I$, we obtain upper bounds for the 9th and 11th terms
%
\begin{equation}
\label{eq_bound_9_11}
\tr{V' D_\kappa V} \le {v'_{ik}} / {(D_\kappa V')_{ik}v_{ik}^2}.
\end{equation}

Then, we find a lower bound for each negative term of \cref{eq_dec_proof_obj}. By using the following inequality $\alpha \ge 1+ \log\alpha$ for $\alpha >0$, it is direct to get
\begin{equation}
\label{eq_dec_proof_inequality}
\frac{v_{ik}}{v'_{ik}}	\ge 1 + \log\frac{v_{ik}}{v'_{ik}},\quad \text{and}\quad
\frac{v_{ik}v_{il}}{v'_{ik}v'_{il}}	\ge 1 + \log\frac{v_{ik}v_{il}}{v'_{ik}v'_{il}}.
\end{equation}
From \cref{eq_dec_proof_inequality}, the lower bounds for the negative terms are:
\begin{equation}
\label{eq_bound_neg}
\begin{aligned}
\tr{V^TB_\kappa^+} &	\ge \sum_{ik}(B_\kappa^+)_{ik}v'_{ik}(1+\log \frac{v_{ik}}{v'_{ik}}),	\\
\tr{VA_\kappa^- V^T} & \ge \sum_{ikl}(A_\kappa^-)_{kl}v'_{ik}v'_{il}(1+\log\frac{v_{ik}v_{il}}{v'_{ik}v'_{il}}),	\\
\tr{V^T W_\kappa V} & \ge \sum_{ikl}(W_\kappa)_{kl}v'_{ki}v'_{li}(1+\log\frac{v_{ki}v_{li}}{v'_{ki}v'_{li}}).	
\end{aligned}
\end{equation}
Collecting all the bounds with their factors, we obtain $J(V,V')$. Therefore, according to \cref{eq_bound_3_7,eq_bound_2_6,eq_bound_9_11,eq_bound_neg}, it is easy to verify that $J(V,V')\ge F(V)$. Let $V'=V$, it is easy to see $J(V,V) = F(V)$. Hence, $J(V,V')$ is an auxiliary function.

Next, we prove the second statement. For $J(V,V')$, we take the first order partial derivative with respect to each $v_{ik}$:
%
\begin{equation}
\label{eq_dec_proof_derivative}
\begin{aligned}
&	\frac{\partial J(V,V')}{\partial v_{ik}} \\
= & - 2(B_1^+)_{ik}\frac{v'_{ik}}{v_{ik}} + 2(B_1^-)_{ik}\frac{v_{ik}}{v'_{ik}} 	\\
&  + 2\frac{(V'A_1^+)_{ik}v_{ik}}{v'_{ik}} - 2\frac{(V'A_1^-)_{ik}v'_{ik}}{v_{ik}} \\
& - 2(B_2^+)_{ik}\frac{v'_{ik}}{v_{ik}}+ 2(B_2^-)_{ik}\frac{v_{ik}}{v'_{ik}}	\\
&  + 2\frac{(V'A_2^+)_{ik}v_{ik}}{v'_{ik}} - 2\frac{(V'A_2^-)_{ik}v'_{ik}}{v_{ik}} 	\end{aligned}\end{equation}\begin{equation}\begin{aligned}\nonumber
&  + 2\lambda_2 \frac{(D_PV')_{ik}v_{ik}}{v'_{ik}} - 2\lambda_2 \frac{(W_PV')_{ik}v'_{ik}}{v_{ik}} \\
& + 2\lambda_2 \frac{(D_QV')_{ik}v_{ik}}{v'_{ik}} - 2\lambda_2 \frac{(W_QV')_{ik}v'_{ik}}{v_{ik}}. 
\end{aligned}
\end{equation}
Then, the Hessian matrix of $J(V,V')$ is
\begin{equation}
\begin{aligned}
	&	\frac{\partial^2 J(V,V')}{\partial v_{ik}\partial v_{jl}}	\\
= 	& 	\delta_{ij}\delta_{kl} \Big( \frac{2(B_1^+)_{ik}v'_{ik}}{v_{ik}^2} + \frac{2(B_1^-)_{ik}}{v'_{ik}} \\
	&	+ \frac{2(V'A_1^+)_{ik}}{v'_{ik}} + \frac{2(V'A_1^-)_{ik}v'_{ik}}{v_{ik}^2} \\
	&	+ \frac{2(B_2^+)_{ik}v'_{ik}}{v_{ik}^2} + \frac{2(B_2^-)_{ik}}{v'_{ik}}  \\
	&	+ \frac{2(V'A_2^+)_{ik}}{v'_{ik}} + \frac{2(V'A_2^-)_{ik}v'_{ik}}{v_{ik}^2}	\\
	&	+ \frac{2\lambda_2 (D_PV')_{ik}}{v'_{ik}} + \frac{4\lambda_2 (W_PV')_{ik}v'_{ik}}{v_{ik}^2}\\
	& 	+ \frac{2\lambda_2 (D_QV')_{ik}}{v'_{ik}}	+ \frac{2\lambda_2 (W_QV')_{ik}v'_{ik}}{v_{ik}^2} \!\Big) \\
 =  &	\frac{2(\Delta_1)_{ik} \! v'_{ik}}{v_{ik}^2} + \frac{2 (\Delta_2)_{ik}}{v'_{ik}},
\end{aligned}
\end{equation}
where $\delta_{ij}=1$ if $i=j$ and 0 otherwise, and 
\begin{equation}
\begin{aligned}
\Delta_1 = &	B_1^+ + B_2^+ + V'(A_1^- + A_2^- ) + \lambda_2 (W_P +  W_Q)V'	\\
\Delta_2 = &	B_1^- + B_2^- + V'(A_1^+ + A_2^+) + \lambda_2 (D_P + D_Q)V'.
\end{aligned}
\end{equation}
Therefore, the Hessian matrix of $J(V,V')$ is diagonal with positive entries, revealing that it is positive definite and $J(V,V')$ is a convex function of $V$. Therefore, the global optimal of $J(V,V')$ is obtained by its first optimality condition $\frac{\partial J(V,V')}{\partial v_{ik}}=0$. According to \cref{eq_dec_proof_derivative}, we get 
%
\begin{equation}
\label{eq_dec_proof_first_condition}
(\Delta_1)_{ij} {v'_{ij}} / {v_{ij}} = (\Delta_2)_{ij} {v_{ij}} / {v'_{ij}},
\end{equation}
which leads to \cref{eq_update_v_proof} and concludes the proof.
\end{proof}
Up to now, we have given the useful definitions, lemmas, and propositions. Next, we will prove \cref{thm_decrease_V}.
\begin{proof}
Let $V^{(t)}$ be the value of $V$ at $t$-th iteration of optimizing \cref{eq_sub_v}. According to \cref{prop_aux} and \reflemma{thm_aux}, it is easy to see that $\{F(V^{(t)})\}$ is monotonically decreasing under the updating \cref{eq_update_v}. 

We introduce the Lagrangian function
\begin{equation}
L(V) = F(V) - \tr{\Psi V^T},
\end{equation}
where the Lagrangian multipliers $\Psi = [\psi_{ij}]$ enforce nonnegative constraints, $v_{ij}\ge 0$. The first order optimality condition gives
\begin{equation}
-2B_1 - 2B_2 + 2V(A_1+A_2) + 2\lambda_2 (L_P+L_Q) V - \Psi = 0.
\end{equation}
From the complementary slackness condition, we obtain
\begin{equation}
\label{eq_fixed_point}
(-2B_1 - 2B_2 + 2V(A_1+A_2) + 2\lambda_2 (L_P+L_Q) V)_{ij} v_{ij} = \psi_{ij}v_{ij} = 0,
\end{equation}
which is a fixed point equation that the solution must satisfy at convergence. It is easy to see that the limiting solution of \cref{eq_update_v} satisfies the fixed point equation in \cref{eq_fixed_point}. At convergence, $V^{(\infty)} = V^{(t+1)} = V^{(t)} = V$, i.e.,
%
\begin{equation}
\label{eq_update_proof}
v_{ij} = v_{ij}
\sqrt{
\frac{
(B_1^{+} + B_2^{+} + V(A_1^{-} + A_2^{-}) + \lambda_2 (W_P + W_Q) V)_{ij}}
{(B_1^{-} + B_2^{-} + V(A_1^{+} + A_2^{+}) + \lambda_2 (D_P + D_Q) V)_{ij}}
}.
\end{equation}
By the definitions in \cref{eq_def_AB}, \cref{eq_update_proof} reduces to
\begin{equation}
\label{eq_update_proof_reduce}
(-2B_1 - 2B_2 + 2V(A_1+A_2) + 2\lambda_2 (L_P+L_Q) V)_{ij} v_{ij}^2=0.
\end{equation}
which is \cref{eq_fixed_point}.
\end{proof}


\ifCLASSOPTIONcaptionsoff
  \newpage
\fi



%

%
%

\bibliographystyle{IEEEtrans}
\bibliography{SemiNMF_2D_arxiv}

%

%
%
%




\end{document}